\documentclass{article}

\usepackage[preprint]{neurips_2026}

\usepackage{microtype}
\usepackage{graphicx}
 
\usepackage{booktabs} 

\usepackage[
    colorlinks=true,
    linkcolor=blue,
    citecolor=blue,
    urlcolor=blue
]{hyperref}
 
\usepackage{etoc}

\usepackage{amsmath}
\usepackage{amssymb}
\usepackage{mathtools}
\usepackage{amsthm}

\usepackage[capitalize,noabbrev]{cleveref}

\theoremstyle{plain}
\newtheorem{theorem}{Theorem}[section]

\newtheorem{lemma}[theorem]{Lemma}

\theoremstyle{definition}

\theoremstyle{remark}
\newtheorem{remark}[theorem]{Remark}

\usepackage{url}            
\usepackage{booktabs}       
\usepackage{amsfonts}       
\usepackage{nicefrac}       
\usepackage{microtype}      
\usepackage{xcolor}         
\usepackage{enumitem}
\usepackage{arydshln}
\usepackage{tabularx}
\usepackage{wrapfig}
\usepackage[skip=5pt]{caption}
 \usepackage{colortbl}
\usepackage{mathrsfs}
\usepackage{mathtools}
\usepackage{placeins} 
\usepackage{algorithm}
\usepackage{algpseudocode}

\usepackage{cleveref}

\crefname{theorem}{Theorem}{Theorems}
\Crefname{theorem}{Theorem}{Theorems}

\crefname{lemma}{Lemma}{Lemmas}
\Crefname{lemma}{Lemma}{Lemmas}

\crefname{proposition}{Proposition}{Propositions}
\Crefname{proposition}{Proposition}{Propositions}

\crefname{corollary}{Corollary}{Corollaries}
\Crefname{corollary}{Corollary}{Corollaries}

\usepackage{bm}
\usepackage{tcolorbox}
\tcbuselibrary{breakable,skins}
\usepackage{dialogue}
\usepackage{subfigure}
\usepackage{adjustbox}
\usepackage{subcaption}
\usepackage{float} %
\usepackage{booktabs}
\usepackage{xspace}
\usepackage{amsmath}      
\usepackage{amsthm}       
\usepackage{amssymb}      
\usepackage{amsfonts}     
\usepackage{mathtools}    
 
\usepackage{arydshln}
\usepackage{graphicx}
 \usepackage{listing}
\usepackage{pifont}
\usepackage{todonotes}
\usepackage{multicol}
\usepackage{multirow}
 
\usepackage{lineno}
\usepackage{xcolor}

\newcommand{\ospoProp}{\textsc{OSPO-Prop}}
\newcommand{\ospoRank}{\textsc{OSPO-Rank}}
\newcommand{\ospoClip}{\textsc{OSPO-Clip}}
\newcommand{\ospo}{\textsc{OSPO}}

\newcommand{\nk}[1]{\textcolor{black}{#1}}

\title{Owen-Shapley Policy Optimization: A Principled RL Algorithm for Generative Search LLMs}

\author{%
\textbf{Abhijnan Nath}$^{1,2}$\begingroup
\thanks{$^{1}$Amazon Science, Seattle, WA, USA \quad
$^{2}$Situated Grounding and Natural Language (SIGNAL) Lab, Colorado State University. Work done during an internship at Amazon Science.}
\endgroup
\quad
\textbf{Alireza Bagheri Garakani}$^{1}$ \quad
\textbf{Tianchen Zhou}$^{1}$ \\
\textbf{Fan Yang}$^{1}$ \quad
\textbf{Yan Gao}$^{1}$ \quad
\textbf{Nikhil Krishnaswamy}$^{2}$ \\
\texttt{\{abhijnan.nath, nkrishna\}@colostate.edu \quad alirezg@amazon.com}
}

\begin{document}

\maketitle

\begin{abstract}
Large language models are increasingly trained via reinforcement learning for personalized recommendation tasks, but standard methods like GRPO rely on sparse, sequence-level rewards. \nk{These obscure which tokens actually contribute to high-quality outputs, creating a \textit{credit assignment gap}}. This gap is especially problematic when models must infer latent user intent from under-specified language without ground truth labels, which is a reasoning pattern rarely seen during pretraining \nk{but commonly required in deployment}. We introduce \textsc{Owen-Shapley Policy Optimization} (\textsc{OSPO}), a framework that redistributes sequence-level advantages based on tokens' \textit{marginal} contributions to outcomes. \textsc{OSPO} \nk{transforms task feedback into} potential-based reward shaping via Shapley-Owen attributions to assign segment-level credit while preserving the optimal policy—all without parametric value models. By forming coalitions of semantically coherent units (e.g., phrases describing product attributes or sentences capturing preferences), \textsc{OSPO} identifies which response parts drive performance. Experiments on Amazon ESCI and H\&M Fashion datasets including controlled generation tasks show consistent gains over baselines and notable test-time robustness to out-of-distribution retrievers unseen during training. Project Code: \url{https://github.com/amazon-science/owen-shapley-policy-optimization}
\end{abstract}

\vspace{-2mm}
\section{Introduction}
\label{sec:intro}
\vspace{-2mm}
 Large language models (LLMs) are increasingly central to modern recommendation pipelines. They can synthesize nuanced user preference summaries, reformulate complex queries, and generate natural-language artifacts that condition downstream retrieval and ranking processes~\citep{peng2025surveyllmpoweredagentsrecommender, lin2025rec}. \nk{This capability is increasingly useful for relevant, personalized recommendations incorporating user history and implicit preferences.} A major factor driving this adoption is their ability to learn directly from explicit task feedback through Reinforcement Learning (RL), bypassing costly human annotations. This makes them especially effective in dynamic, personalized settings where fine-grained supervision is infeasible~\citep{lin2025rec,lightman2023let,cheng2025stopsummationminformcredit}. Efficient algorithms such as Group-Relative Policy Optimization (GRPO;~\citet{shao2024deepseekmath}) enable LLMs to optimize against verifiable sequence-level signals from black-box components such as retrievers, rankers, or quality assessors, without requiring an external value model.

   \begin{figure*}[t]
  \centering
  \includegraphics[width=\textwidth]{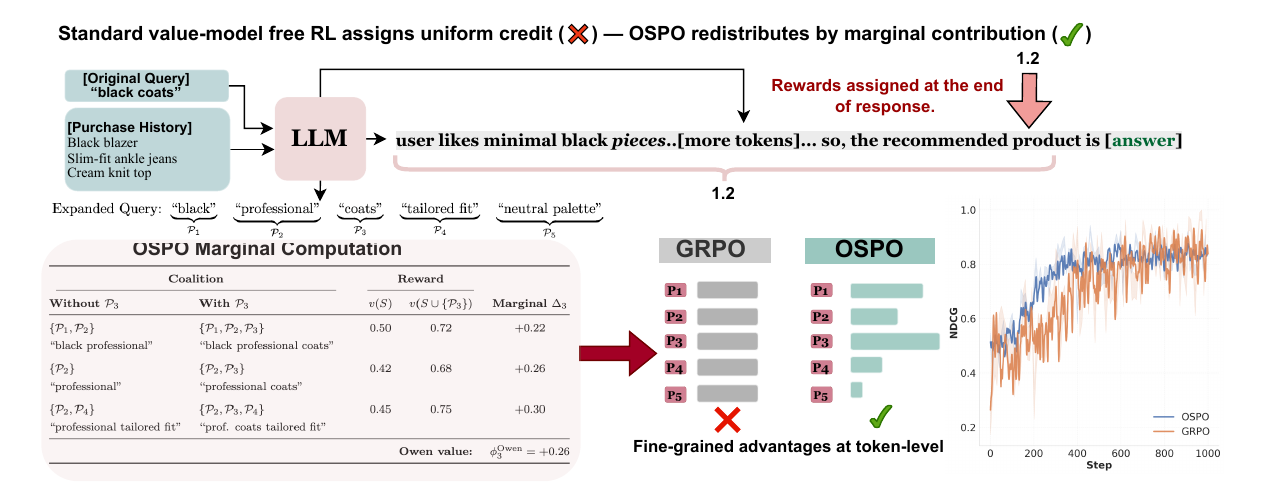}
  \caption{\textbf{OSPO overview: fine-grained credit assignment via Owen-Shapley values.} Standard value-model-free RL (e.g., GRPO;~\citet{shao2024deepseekmath}) assigns uniform advantages to all tokens {\it via a single terminal reward} (bottom center right, grey bars), ignoring segment-level contributions. OSPO evaluates contiguous coalitions by querying a retriever or reward model with partial sequences, computing each segment's marginal contribution to retrieval quality. ``\ospo{} Marginal Contribution'' shows coalition sampling for segment $\mathcal{P}_3$ (``coats''): adding $\mathcal{P}_3$ to coalition $\{\mathcal{P}_1, \mathcal{P}_2\}$ increases the reward from 0.50 to 0.72, yielding marginal $\Delta_3 = +0.22$; averaging such marginals across all coalitions produces the Owen value $\phi_3^{\text{Owen}} = +0.26$, indicating that ``coats'' consistently improves retrieval by providing explicit category information. These Owen values are then used to redistribute advantages to tokens proportionally to their impact (bottom center right, teal bars), focusing gradient updates on truly high-contribution segments. Training dynamics (bottom right) show \ospo{} reaches target performance in $\sim$50\% fewer steps than GRPO, demonstrating improved sample efficiency.}
  \vspace{-5mm}
  \label{fig:ospo_overview}
\end{figure*}

However, removing the value model introduces structural limitations in how feedback is represented and propagated during policy training. In single-turn\footnote{Single-turn contrasts with multi-turn search, where the LLM incrementally refines its reasoning through interactive context acquisition~\citep{yao2022react,xi2025surveyllmbaseddeepsearch}.} generative recommendation systems, two major issues emerge. First, without a value model to decompose feedback across the sequence, GRPO-style training assigns a single scalar advantage to the \textit{entire} response, obscuring which phrases or reasoning segments actually drive downstream utility~\citep{yue2025vapo}, and limiting interpretability and efficiency~\citep{henderson2018deep, shao2024deepseekmath, cao2025scar, cheng2025stopsummationminformcredit}. Second, the absence of a structured value function increases vulnerability to reward hacking, where the model overfits to surface-level cues in black-box retrievers or rankers that expose only scalar feedback rather than the internal semantics of task success~\citep{stiennon2020learning, skalse2025definingcharacterizingrewardhacking, lin2025rec}. The uniform advantages of GRPO training cannot distinguish useful segments from exploitative ones, allowing shortcuts to dominate gradients. Given the rapid adoption of value-model-free methods~\citep{zhang2025towards}, 
resolving these limitations \textit{within} the critic-free regime is essential—especially where ground-truth preferences are latent and 
feedback comes exclusively from black-box retrievers.

To address these limitations, we introduce \textbf{Owen–Shapley Policy Optimization} (\textsc{OSPO}), an attribution-based RL framework that redefines how gradients are distributed across generated tokens. Unlike standard policy gradients that treat all tokens uniformly~\citep{williams92, ahmadian-etal-2024-back, shao2024deepseekmath}, \textsc{OSPO} asks: \emph{which parts of the sequence truly earned the reward?} Leveraging Owen–Shapley values from cooperative game theory~\citep{shapley1953stochastic, owen1977values, winter1992consistency, ruess2024fairnessanalysisshapleyoweneffects}, \textsc{OSPO} models semantically coherent token segments as players in a coalition game, estimating their true contribution to the observed reward.

For example, in product search, customer queries often contain attribute-specific preference signals embedded in phrases, such as “looking for [product] in \textbf{my usual neutral tones}”. \ospo{} treats these segments as high-level abstractions in the action space, akin to options~\citep{sutton1984temporal} or macro-actions~\citep{chai2025marlhfreinforcementlearninghuman} in hierarchical RL. To assign such values, each segment's marginal contribution to the total reward is measured by how much it increases the reward when added to different coalitions. Segments that consistently improve outcomes receive credit, while those that dilute or distract are down-weighted. These attributions are obtained by querying a feedback model that returns scalar scores (e.g., a retriever or reward model), requiring no external value networks or process reward models~\citep{schulman2017proximalpolicyoptimizationalgorithms,lightman2023let}.

Our key contribution is proposing a policy learning approach that can integrate such fine-grained value attributions from Owen values in learning good policies, while additionally leveraging Monte Carlo-based sampling rollouts that allow \ospo{} to achieve a reasonable tradeoff between exploration and exploitation~\citep{sutton2018reinforcement}. \textbf{Our core insight in achieving this balance is one of \emph{gradient budget reallocation via advantage redistribution}}: rather than spreading policy gradients uniformly across all tokens in a sequence~\citep{shao2024deepseekmath}, \textsc{OSPO} concentrates updates on high-attribution segments while reducing signal on low-value content. Fig.~\ref{fig:ospo_overview} illustrates \ospo{}'s advantage redistribution mechanism. \ospo{} computes fine-grained advantages for response segments directly from task-based feedback or a reward model—weighted by the Owen attributions. Specifically, \textsc{OSPO} applies a simple mapping from segment-level Owen values to tokens that respects the scale of standard Monte Carlo group-based advantages~\citep{williams92}, while reducing computation complexity involved in exact Shapley computations. This helps resolve the credit assignment gap inherent in learning \textit{solely} from terminal rewards \textit{without compromising learning of final outcomes.} To stabilize training dynamics, \textsc{OSPO} admits potential-based reward-shaping (PBRS;~\citet{ng1999policy}) to adapt segment-level Owen values to tokens by adopting the standard clip-surrogate objective for stable policy updates~\citep{schulman2017proximalpolicyoptimizationalgorithms}. 
As Fig.~\ref{fig:ospo_overview} demonstrates, this focused attribution yields substantially faster learning similar to prior work~ \citep{koo2025learning}: \textsc{OSPO} reaches 85\% Normalized Discounted Cumulative Gain (NDCG) in approximately 400 training steps on the fashion domain search task, while GRPO requires over 800 steps to approach similar performance despite identical training data and settings.

\vspace{-2mm}
\section{Related Work}
\vspace{-2mm}

\paragraph{Generative Recommendation Systems}
LLMs have transformed recommendation systems through reasoning-driven decision making~\citep{wang2020kerl, zhao2018recommendations, liu2023exploration}. Recent methods integrate LLMs via feature enrichment~\citep{xi2024towards, li2023ecomgptinstructiontuninglargelanguage}, ranking~\citep{geng2023recommendationlanguageprocessingrlp, luo2024recrankerinstructiontuninglarge}, and conversational interaction~\citep{zhou-etal-2020-towards, gao2023chatrecinteractiveexplainablellmsaugmented, bang2025llm}. RL-LLM hybrids~\citep{jeong2023factual, sun2024rlrf4rec, lu2024aligning, rafailov2024direct} enable personalization but focus on offline tuning. Building on Rec-R1~\citep{lin2025rec}, which uses online task-specific feedback, \textsc{OSPO} addresses sparse credit assignment via Owen-Shapley attributions over segment-level spans. 

 \vspace{-3mm}
\paragraph{Owen–Shapley Values} Owen–Shapley values have been explored across domains~\citep{owen2014sobol, gimenez2019owen, lundberg2017unified} and applied in NLP for LLM interpretation~\citep{goldshmidt2024tokenshap, mohammadi2024explaining}, prompt valuation~\citep{liu2023prompt}, model compression~\citep{sun2025efficient}, and dataset refinement~\citep{he2024shed}. However, their use in RLHF remains limited. Recent work like SCAR~\citep{cao2025scar} and \citet{koo2025learning} extend Shapley-based reward distribution for LLM training.

 \vspace{-3mm}
\paragraph{Densifying Sparse Rewards}
Dense feedback in policy optimization traditionally relies on value function-based advantage estimation~\citep{sutton1999between, schulman2015high, schulman2017proximalpolicyoptimizationalgorithms}, which is computationally heavy and requires learned critics. Value model-free alternatives like DPO and GRPO~\citep{rafailov2024direct, shao2024deepseekmath, liu2025understandingr1zeroliketrainingcritical} reduce overhead, but assign uniform credit to tokens~\citep{parthasarathi2025grpolambdacreditassignmentimproves}, often causing to overfitting~\citep{yu2025dapo, guo2025segment}. Process Reward Models~\citep{lightman2023let, wu2023fine, bai2022training, choudhury2025processrewardmodelsllm} densify rewards via step-level supervision but require labeled intermediate steps and remain confined to deterministic domains~\citep{fang2025mathodyssey, sun2025challengingboundariesreasoningolympiadlevel}. While prior work addresses uniform credit assignment through temporal propagation via mid-sequence rollouts~\citep{kazemnejad2024vineppo, guo2025segment} or eligibility traces on the importance ratio~\citep{parthasarathi2025grpolambdacreditassignmentimproves}—leaving the advantage function uniform—and~\citet{samineni2025rlonlyanalyzingstructural} show that uniform advantage assignment reduces GRPO to filtered imitation learning, \textsc{OSPO} directly redistributes the advantage itself proportional to each segment's marginal retrieval contribution without mid-sequence rollouts or a value model. Attention-based methods~\citep{chan2024dense,ng1999policy} propagate reward through attention weights requiring tokenizer alignment but can entangle spurious correlations~\citep{cao2025scar,xie2025capoenhancingllmreasoning}.  Recent Shapley-based approaches such as SCAR~\citep{cao2025scar} compute principled marginal contributions but operate within PPO's actor-critic framework, requiring a learned value model and a reward model capable of scoring partial completions—a condition standard preference RMs~\citep{lambert2024rewardbench} do not satisfy. \ospo{} instead targets the critic-free regime with retrieval-based rewards, where dense similarity models naturally admit partial phrase coalitions as valid inputs~\citep{reddy2022shopping}, operating directly over GRPO-style group advantages without an auxiliary value network.

\vspace{-2mm}
\section{Problem Formulation and Background}
\label{sec:problem}
\vspace{-2mm}

We model text generation for recommendation as a Markov Decision Process (MDP) defined by the tuple $(\mathcal{S}, \mathcal{A}, P, r, \rho_0, \gamma)$. \textbf{State space} $\mathcal{S}$ consists of all partial sequences $(x, y_{<t})$ combining the input context $x$ (e.g., user purchase history or raw search query) and tokens generated so far $y_{<t} = (y_1, \ldots, y_{t-1})$. \textbf{Action space} $\mathcal{A}$ is the vocabulary from which the policy selects the next token $y_t$. \textbf{Transition function} $P(s_{t+1} | s_t, a_t)$ is deterministic—appending token $a_t = y_t$ to the history to form state $s_{t+1} = (x, y_{\leq t})$. \textbf{Reward function} $r: \mathcal{S} \times \mathcal{A} \to \mathbb{R}$ assigns scalar feedback. \textbf{Initial state distribution} $\rho_0$ samples input $x \sim \mathcal{D}$ from the data distribution. \textbf{Discount factor} $\gamma = 1$ for the undiscounted episodic setting as is standard for LLMs~\citep{rafailov2024r}. A trajectory $\tau = (s_0, a_0, s_1, a_1, \ldots, s_T, a_T)$ corresponds to autoregressively generating a complete sequence $y = (y_1, \ldots, y_T)$ via the policy $\pi_\theta(y_t | s_t) = \pi_\theta(y_t | x, y_{<t})$, where $\pi_\theta(y | x) = \prod_{t=1}^T \pi_\theta(y_t | x, y_{<t})$. 

The learning objective\footnote{We set the KL coefficient to $\beta=0$, following prior work~\citep{lin2025rec}.} is to find a policy that maximizes expected downstream performance:
\begin{equation}
\label{eq:objective}
\max_{\theta}\; \mathcal{J}(\theta) = \mathbb{E}_{x \sim \mathcal{D},~ y \sim \pi_{\theta}(\cdot | x)}[\, r(x,y) \,],
\end{equation}
where $r(x, y) \in \mathbb{R}$ is a terminal reward from the downstream system—for instance, NDCG in product search or a reward model score in user profile generation~\citep{zheng2023judgingllmasajudgemtbenchchatbot}. 

Because $r(x,y)$ is \textbf{sequence-level and terminal}, with $r_t = 0$ for $t<T$ and $r_T = r(x,y)$, the policy receives no signal about which tokens—or, more meaningfully, which \emph{semantic segments} of the response—contributed to success, creating the credit assignment problem. To address sparse feedback, methods like GRPO~\citep{shao2024deepseekmath, lin2025rec} replace the learned value function with a sampling-based baseline. GRPO samples $G$ responses per prompt and computes group-relative advantages by normalizing rewards within the batch:
\vspace{-4mm}
\begin{equation}
\label{eq:grpo_advantage}
\hat{A}^{(g)} = \frac{r^{(g)} - \bar{r}}{\sigma_r}, 
\quad \text{where} \quad 
\bar{r} = \tfrac{1}{G}\sum_{g'=1}^G r^{(g')}, \quad \sigma_r = \sqrt{\tfrac{1}{G}\sum_{g'=1}^G (r^{(g')} - \bar{r})^2}.
\end{equation}
\vspace{-4mm}
The GRPO objective is then:

\begin{equation}
\label{eq:grpo_obj}
\mathcal{J}^{\text{GRPO}}(\theta)
= \frac{1}{G}\sum_{g=1}^G \frac{1}{T}\sum_{t=1}^{T}
\min\!\Big(
\rho_t^{(g)} \hat{A}^{(g)},
\operatorname{clip}(\rho_t^{(g)},\,1-\epsilon,\,1+\epsilon)\hat{A}^{(g)}
\Big),
\end{equation}

where $\rho_t^{(g)}=\frac{\pi_\theta(y_t^{(g)} \mid x, y_{<t}^{(g)})}{\pi_{\theta_{\text{old}}}(y_t^{(g)} \mid x, y_{<t}^{(g)})}$ is the importance sampling ratio. 
While GRPO reduces variance via group normalization, it applies a \emph{uniform advantage} $\hat{A}^{(g)}$ to all tokens in a sequence~\citep{samineni2025rlonlyanalyzingstructural}, obscuring which tokens or semantic segments drove the final reward $r^{(g)}$.  


\vspace{-2mm}
\subsection{Owen-Shapley Policy Optimization: Segment-Level Credit Assignment}
\label{ssec:owen-shap-values}
\vspace{-2mm}

The mismatch between GRPO’s token-level uniform advantages and the compositional structure of recommendation tasks motivates a shift in granularity. To resolve this, our approach, \textsc{OSPO}, leverages the structure of Owen “coalitions”~\citep{owen1977values}—groups of segments (or, in game-theoretic terms, players)—that compute what value \textit{should} be assigned to each segment based on its \textit{expected marginal contribution} to the total reward. Intuitively, this allows the policy $\pi_\theta$ to learn which among the segments it generated actually led to the observed feedback. In other words, the policy learns counterfactually—reasoning over which parts of its response were relatively “more valuable” than others.\footnote{“Counterfactual” here refers to retrieval marginal contribution, not what the policy would have generated under different conditions~\citep{wangroad}.}

Specifically, let a response $y^{(g)}$ generated from the policy $\pi_\theta$ during group-based sampling be decomposed into $N$ semantically meaningful segments $\mathcal{A}^{(g)} = \{a_1, \dots, a_N\}$, where each $a_j$ denotes a textual span corresponding to a linguistically coherent unit—such spans may represent noun phrases, adjective-noun pairs, verb phrases, or even full sentences depending on the task (see Appendix~\ref{ssec:coalition_ablations} for extraction details). For any coalition $S \subseteq \{1, \ldots, N\}$, we define the \emph{characteristic function} $v(S)$ as the reward obtained when only the segments indexed by $S$ are present. To implement this, we construct a partial query $y_S$ by concatenating the segments in coalition $S$ while keeping the input context $x$ fixed: $y_S = a_i \oplus a_{i+1} \oplus \cdots$ for $i \in S$, where $\oplus$ denotes string concatenation.

Classical Shapley values~\citep{shapley1953stochastic} assign to each segment $j \in \{1, \ldots, N\}$ its average marginal contribution across all possible subsets:
\vspace{-2mm}
\begin{equation}
\small
\label{eq:shapley}
\phi_j
~=~
\sum_{S \subseteq \{1,\ldots,N\} \setminus \{j\}} 
\frac{|S|!\,(N-|S|-1)!}{N!}\,\Big( v(S \cup \{j\}) - v(S) \Big).
\end{equation}
This formulation ensures key fairness properties—\emph{efficiency} ($\sum_j \phi_j = v(\{1,\ldots,N\})$), \emph{symmetry}, and \emph{linearity}—but quickly becomes intractable as $N$ grows, requiring evaluation of $2^N$ coalitions. The Owen value~\citep{owen1977values} extends Shapley's framework by restricting the coalitions to a structured subset while preserving the core fairness guarantees~\citep{owen2014sobol}. In our case, we impose a \textbf{contiguity constraint}: we only evaluate coalitions $S = \{i, i+1, \ldots, i+w-1\}$ consisting of consecutive segments, where $w \leq w_{\max}$ is a maximum coalition width hyperparameter. This restriction serves two purposes: (1) it reduces computational complexity from $O(2^N)$ to $O(N \cdot w_{\max})$—\textbf{linear} in the number of segments for fixed $w_{\max}$ compared to the quadratic complexity of prior Shapley approximations~\citep{cao2025scar}—making online evaluation tractable during RL training, and (2) it preserves semantic coherence, as non-contiguous coalitions (e.g., skipping middle segments) would produce fragmented queries that may confound reward evaluation. We denote the resulting segment attributions as $\phi_j^{\text{Owen}}$, computed by averaging marginal contributions $v(S \cup \{j\}) - v(S)$ only over contiguous coalitions $S$ where $j \notin S$ but $S \cup \{j\}$ remains contiguous:
\begin{equation}
\small
\label{eq:owen_contiguous}
\phi_j^{\text{Owen}}
~=~
\frac{1}{|C_j|}\sum_{S \in C_j} \Big( v(S \cup \{j\}) - v(S) \Big),
\end{equation}
where $C_j = \{S : S \text{ contiguous}, j \notin S, S \cup \{j\} \text{ contiguous}\}$ is the set of applicable coalitions for segment $j$. Full coalition generation details and empirical verification of linear scaling are in Appendix~\ref{ssec:coalition_formation} (Fig.~\ref{fig:complexity}).

This coalitional view naturally aligns with the linguistic structure of LLM outputs in recommendation tasks, where contiguous spans, such as attribute–intent phrases (``blue midi dress'') or preference clauses (``prefers seasonal styles''), form coherent compositional units. By assigning credit to these coalitions, \textsc{OSPO} captures\footnote{Empirically, we find that Owen values correlate $3.3\times$ more strongly with 
ground-truth solo-phrase retrieval quality than attention-based 
proxies ($\rho{=}0.289$ vs.\ $0.078$) and outperform TF-IDF 
attribution ($\rho{=}0.079$); see 
Appendix~\ref{ssec:compute_ospo_appendix}.}  how phrase combinations jointly influence downstream reward without enumerating all $2^N$ subsets. Conceptually, this mirrors \emph{options} or \emph{temporal abstractions} in reinforcement learning~\citep{sutton2018reinforcement,konidaris2009skill,chai2025marlhfreinforcementlearninghuman}, allowing more efficient exploration and faster credit propagation across a response.

\vspace{-4mm}
 

\paragraph{From Segment Attributions to Token Advantages} 
While Owen values $\phi_j^{\text{Owen}}$ quantify segment-level contributions, we must bridge two granularity gaps to effectively use such fine-grained values in policy optimization. First, standard baseline-based advantage computation requires exploring multiple completions per prompt via Monte Carlo sampling~\citep{williams92, kool2019buy} to estimate which responses are better than average. This yields per-sequence advantages $\hat{A}^{(g)} = R^{(g)} - \bar{R}$ that indicate whether completion $g$ should be reinforced or suppressed. Second, clipped surrogate objectives~\citep{schulman2017proximalpolicyoptimizationalgorithms, shao2024deepseekmath, yue2025vapo} operate at the “token” level, requiring decomposition of sequence-level feedback into token-level signals. Owen values address \textit{within-sequence} attribution by identifying which segments contributed to quality, but the policy must also learn from these \textit{between-sequence} group-level advantages to effectively explore and distinguish better completions from worse ones. 

To achieve both objectives, our core insight is \textbf{advantage redistribution}: instead of uniform assignment, we redistribute the per-sequence advantages $\hat{A}^{(g)}$ across tokens, weighted by their Owen values as computed above. This strategy allows us to leverage group-based advantages for variance reduction~\citep{williams92, kool2019buy} during exploration while learning fine-grained token attributions, and overcoming the credit assignment problem \textit{without relying on an external value model}. We construct a mapping $W \in \mathbb{R}^{N \times T}$ where $W_{it}$ encodes  
fraction of token span covered by the segment $a_i$ and token $y_t$, projecting segment Owen values to token space: $\boldsymbol{\phi}_{\text{tok}} = W^\top \boldsymbol{\phi}_{\text{seg}}$. After normalizing to unit sum $\tilde{\phi}_t = \phi_t / \sum_{t'} \phi_{t'}$, a naive proportional redistribution $\tilde{A}_t = \hat{A}^{(g)} \cdot \tilde{\phi}_t$ would create severe length bias~\citep{liu2025understandingr1zeroliketrainingcritical, yue2025vapo} where shorter responses are rewarded more while longer incorrect responses go under-penalized. Since Owen values must sum to the total reward (due to the efficiency property), longer sequences experience systematic dilution where each token receives exponentially weaker gradient signals compared to shorter responses. To eliminate this bias, we multiply by sequence length:
\vspace{-1mm}
\begin{equation}
\label{eq:ospo_advantage}
A_t^{(g)} = T \cdot \tilde{\phi}_t^{(g)} \cdot \hat{A}^{(g)}.
\end{equation}

This ensures the average token advantage $\frac{1}{T}\sum_t A_t^{(g)} = \hat{A}^{(g)}$ equals the original sequence-level advantage regardless of length (proof in \Cref{lemma:length_invariant}, Appendix~\ref{app:proofs}), eliminating the length-dependent training bias identified by~\citet{liu2025understandingr1zeroliketrainingcritical}.


Substituting these \textbf{token-specific} advantages $A_t^{(g)}$ into the GRPO objective (Eq.~\ref{eq:grpo_obj})—which uses uniform sequence-level advantages $\hat{A}^{(g)}$—yields the \textbf{\textsc{OSPO} objective}:
\begin{equation}
\label{eq:ospo_obj}
\mathcal{J}^{\text{OSPO}}(\theta)
= \frac{1}{G}\sum_{g=1}^G \frac{1}{T}\sum_{t=1}^{T}
\min\!\Big(
\rho_t^{(g)} A_t^{(g)},
\operatorname{clip}(\rho_t^{(g)},\,1-\epsilon,\,1+\epsilon) A_t^{(g)}
\Big),
\end{equation}
where $A_t^{(g)} = T \cdot \tilde{\phi}_t^{(g)} \cdot \hat{A}^{(g)}$ 
are Owen-weighted token advantages (cf.\ Eq.~\ref{eq:ospo_advantage}). 
\textbf{The key distinction from GRPO}: while GRPO assigns the same 
advantage $\hat{A}^{(g)}$ to all tokens in sequence $g$, \ospo{} 
redistributes this advantage proportionally to each token's causal 
contribution $\phi_t^{(g)}$ as measured by Owen values—we call this \textbf{\ospoProp{}}. We additionally evaluate \textbf{\ospoRank{}}, 
which maps Owen values to ordinal ranks before redistribution 
($A_t^{(g)} = T \cdot \hat{A}^{(g)} \cdot \text{rank}_t^{(g)} / 
\sum_{t'} \text{rank}_{t'}^{(g)}$) for robustness under noisy or 
skewed attributions from empirical Owen-value estimates, and \textbf{\ospoClip{}}, which constrains redistributed advantages to $A_t \in [-2\sigma_{\hat{A}}, +2\sigma_{\hat{A}}]$ based on Chebyshev's inequality, where $\sigma_{\hat{A}} = \text{std}(\hat{A}^{(g)})$ is the GRPO baseline standard deviation; both preserving length-invariance 
(Eq.~\ref{eq:ospo_advantage}, Appendix~\ref{app:proofs}). 
Algorithm~\ref{alg:ospo} provides the complete \textsc{OSPO} 
training procedure.

\vspace{-2mm}
\section{Experimental Setup}
\label{sec:experimental_setup}
\vspace{-2mm}
Our experimental setup is designed to answer \textbf{three} research questions: 
\textbf{[RQ1]} Can segment-level marginal credit assignment—\textit{without} a value model—improve policy learning in outcome-based retrieval settings where terminal rewards are compositional over semantically coherent response segments?
\textbf{[RQ2]} Does OSPO generalize to a structurally different generation 
task where the attribution target is an intermediate reasoning trace and 
the reward derives from a parametric model rather than a black-box retriever?
\textbf{[RQ3]} Does coalition-based credit assignment learn retriever-agnostic 
query representations that transfer under distribution shift to unseen 
retrievers at test time?
 
\vspace{-2mm}
\paragraph{Tasks and Datasets} We evaluate \ospo{} on two e-commerce tasks testing complementary aspects of advantage redistribution. \textbf{Product search query refinement} addresses \textbf{RQ1} and tests whether models can identify which tokens in expanded queries drive retrieval success. Customers provide under-specified queries that LLMs expand with attributes, but standard RL methods assign uniform credit across all tokens despite unequal contributions to retrieval performance. \nk{Successful \ospo{} should promote those coalitions truly key to successful retrieval.} \textbf{User profile summarization} addresses \textbf{RQ2} and evaluates whether advantage redistribution improves intermediate reasoning quality. Models generate purchase history summaries before predicting next items. Since true preferences are latent, summary quality directly determines downstream accuracy. A key design choice: Owen values target the refined query for search tasks but the reasoning trace for summarization, assessing OSPO's impact across different output types. We use \textbf{ESCI shopping queries}~\citep{reddy2022shopping} for conventional search and \textbf{H\&M Fashion Recommendations}~\citep{h-and-m-personalized-fashion-recommendations} for contextualized search and user summarization. Similar to prior work~\citep{lin2025rec}, all search tasks use dense retrieval via FAISS~\citep{douze2024faiss} with domain-specific encoders (\textsc{all-mpnet-base-v2} for ESCI, \textsc{SIMCSE-Large} for H\&M). Additionally, swapping out the retrievers for product-search during evaluation (post-training) directly addresses \textbf{RQ3}. Models generate chain-of-thought reasoning followed by task outputs;  see Figs.~\ref{fig:esci_search_prompt}, \ref{fig:hnm_fashion_prompt} and \Cref{fig:hnm_summary_fullprompt} in Appendix~\ref{app:baselines} for detailed prompts.

\vspace{-2mm}
\paragraph{H\&M Expert Data Generation}

Since H\&M lacks explicit queries and user summaries, we generate synthetic training data using \texttt{Claude Sonnet~3.0} as an expert model: naturalistic history-grounded search queries for the product search task (Fig.~\ref{fig:hnm_prompt_example}), and chronologically-ordered next-purchase predictions defining winning/losing summaries for preference-tuning methods (\Cref{fig:hnm_summary_fullprompt}). See Appendix~\ref{app:expert_data_generation_prompts} for customer data sampling, prompts and expert validation including comparison of generated query distributions against ESCI and Amazon C4~\citep{hou2024bridging}.

\vspace{-2mm}
\paragraph{Baselines}
We evaluate \textsc{OSPO} against standard finetuning and RL baselines—\textbf{SFT}, \textbf{DPO}~\citep{rafailov2024direct}, and \textbf{GRPO}~\citep{shao2024deepseekmath}—using the same \textbf{Qwen2.5-Instruct 7B} base model across tasks.
SFT is trained on high-quality expert samples filtered by task performance (NDCG $\geq 0.3$ for search, $\geq 0.5$ for summarization), while DPO uses expert contrastive pairs differing by at least 5\% in average precision. GRPO serves as the reference value-free RL baseline~\citep{williams92} using the \textit{same} reward functions as \ospo{}—NDCG of the output refined query for search tasks using domain-specific retrievers, and rewards from a trained Bradley-Terry (BT)~\citep{BradleyTerry1952, Hong2024ORPOMP} loss-based reward model for the user profile summarization task. For additional 
credit assignment comparisons within our compute budget, we evaluate 
\textbf{GRPO-$\lambda$}~\citep{parthasarathi2025grpolambdacreditassignmentimproves} which propagates temporal 
credit via eligibility traces without a value model, and \textbf{PPO}~\citep{schulman2017proximalpolicyoptimizationalgorithms} that uses a learned critic 
at \textbf{1.5B} and \textbf{3B} scales on H\&M product search, and include \textbf{GPT-4o}~\citep{openai2024gpt4ocard} as a zero-shot upper bound.


 \vspace{-2mm}
\paragraph{Evaluation Settings}
We sample responses from trained baselines (incl. base models) on each task’s evaluation set using temperature $1.0$. For product search, performance is measured with retrieval metrics such as NDCG and average precision (AP) with candidate filtering from a pool\footnote{We provide details on the candidate pool construction in Appendix~\ref{app:baselines}.} after initial FAISS-based retrieval. For user profile summarization, we assess sampled summaries via an LLM judge (Qwen-3-Nemotron-32B-Reward;~\citet{wang2025helpsteer3preferenceopenhumanannotatedpreference}) using pairwise head-to-head win/loss/tie evaluations, applying a Borda-count~\citep{emerson2013original} that assigns $1$ for a win and $\tfrac{1}{2}$ for a tie.

\vspace{-2mm}
\section{Results}
\vspace{-2mm}
\begin{table*}[t]
\centering
\small
\setlength{\tabcolsep}{4.5pt}
\begin{tabular}{lcccc|cccc}
\toprule
\multirow{2}{*}{\textbf{Baseline}} & \multicolumn{4}{c|}{\textbf{ESCI Shopping Queries}} & \multicolumn{4}{c}{\textbf{H\&M Fashion}} \\
\cmidrule(lr){2-5} \cmidrule(lr){6-9}
 & \textbf{NDCG} & \textbf{AP} & \textbf{MRR} & \textbf{Recall} & \textbf{NDCG} & \textbf{AP} & \textbf{MRR} & \textbf{Recall} \\
\midrule

Qwen2.5-0.5B                 & 0.379 & 0.115 & 0.392 & 0.116 & 0.337 & 0.320 & 0.320 & 0.387 \\
Qwen2.5-1.5B                 & 0.439 & 0.137 & 0.454 & 0.138 & 0.348 & 0.330 & 0.330 & 0.402 \\
Qwen2.5-3B                   & 0.373 & 0.114 & 0.387 & 0.114 & 0.332 & 0.317 & 0.317 & 0.377 \\
Qwen2.5-32B                  & 0.496 & 0.158 & 0.514 & 0.160 & \underline{0.438} & \underline{0.421} & \underline{0.421} & \underline{0.490} \\
Qwen2.5-72B                  & \underline{0.543} & \underline{0.170} & \underline{0.564} & \underline{0.171} & 0.357 & 0.340 & 0.340 & 0.409 \\
\midrule
SFT (7B)                     & 0.398 & 0.119 & 0.411 & 0.120 & 0.373 & 0.355 & 0.355 & 0.424 \\
DPO (7B)                     & 0.431 & 0.133 & 0.441 & 0.133 & 0.396 & 0.378 & 0.378 & 0.449 \\
GRPO (7B)                    & 0.418 & 0.130 & 0.430 & 0.131 & 0.379 & 0.366 & 0.366 & 0.421 \\
\midrule
\textbf{\ospoRank{}} (7B) & 0.485 & 0.158 & 0.500 & 0.159 & 0.399 & 0.383 & 0.383 & 0.445 \\

\textbf{\ospoClip{}} (7B)  & 0.495& 0.157 & 0.508 & 0.159 & 0.403 & 0.389 & 0.389 & 0.446 \\
\textbf{\ospoProp{}} (7B)  & \textbf{0.522} & \textbf{0.166} & \textbf{0.541} & \textbf{0.167}  & \textbf{0.436} & \textbf{0.419} & \textbf{0.419} & \textbf{0.486} \\
 
\bottomrule
\end{tabular}
\caption{\textbf{Results on ESCI~\citep{reddy2022shopping} and H\&M Fashion Domains~\citep{h-and-m-personalized-fashion-recommendations}}. For all metrics, top-$K{=}1000$ retrieval is performed with candidate pools used for filtering relevant candidates post FAISS-based retrieval. We use a temperature of $1.0$ with sampling to sample all completions. \textbf{Bold} denotes best among 7B fine-tuned models; \underline{underline} denotes best overall.}
\label{tab:combined_search_results_esci_hnm}
\vspace{-5mm}
\end{table*}



Results in Table~\ref{tab:combined_search_results_esci_hnm} show performance metrics of baselines and \ospo{} variants on product search tasks. \ospoProp{} achieves the best performance among 7B models on all metrics, excelling on ESCI (NDCG 0.522, MRR 0.541) and performing strongly on the harder H\&M Fashion dataset (NDCG 0.436, MRR 0.419). Importantly, scaling alone is insufficient: \ospoProp{} at 7B surpasses the 32B model on ESCI (0.522 vs.\ 0.496) and nearly matches it on H\&M (0.436 vs.\ 0.438), while the 72B variant degrades sharply on H\&M (0.357), approaching but not matching \ospoProp{} even with 10$\times$ more parameters.  These gains hold across retrieval architectures, highlighting \ospo{}'s adaptability to different embedding spaces and query structures. While all models perform worse on H\&M due to its higher complexity—requiring integration of latent user preferences—\ospo{} retains its lead, indicating that coalition-based credit assignment effectively captures how query phrasing drives retrieval in both simple and context-rich settings.
 
\vspace{-2mm}
\paragraph{Comparison with Online RL and Offline Methods}
Compared to the online RL baseline GRPO, \ospoProp{} achieves 24.9\% higher NDCG on ESCI (0.522 vs.~0.418) and 15.0\% higher on H\&M (0.436 vs.~0.379). This improvement highlights \ospo{}’s advantage redistribution, which attributes reward more precisely to influential query segments. Over offline alignment methods, \ospo{} outperforms DPO by 21.1\% on ESCI (0.522 vs.~0.431) and 10.1\% on H\&M (0.436 vs.~0.396). While SFT and DPO improve upon base models, their lack of retrieval-grounded feedback limits active exploration and the ability to discern which query elements affect ranking. These shortcomings become more pronounced in the H\&M setting, where relevance depends on modeling user-specific preferences, often “latent” in the input query by design (Appendix~\ref{app:expert_data_generation_prompts}). \ospo{} bridges this gap through principled credit assignment, offering a clear advantage in aligning LLMs for retrieval.

\paragraph{Value Model-Free Marginal Attribution vs. Value Model-Based Credit}
Table~\ref{tab:hnm_wrap_credit_assignment} compares credit assignment 
methods across 1.5B and 3B scales on H\&M product search, with all RL 
baselines initialized from the \textit{same} SFT checkpoint. \begin{wraptable}{r}{0.45\textwidth}
\vspace{-12pt}
\centering
\small
\setlength{\tabcolsep}{4pt}
\renewcommand{\arraystretch}{0.95}

\begin{tabular}{lcccc}
\toprule
\textbf{Model} & \textbf{NDCG} & \textbf{AP} & \textbf{MRR} & \textbf{Recall} \\
\midrule
\multicolumn{5}{l}{\textit{Zero-shot}} \\
Qwen2.5-1.5B  & 0.304 & 0.287 & 0.287 & 0.353 \\
Qwen2.5-3B    & 0.384 & 0.367 & 0.367 & 0.434 \\
GPT-4o & \textbf{0.433} & \textbf{0.413} & \textbf{0.413} & \textbf{0.495} \\
\midrule
\multicolumn{5}{l}{\textit{1.5B fine-tuned}} \\
SFT           & 0.306 & 0.289 & 0.289 & 0.356 \\
GRPO          & 0.307 & 0.289 & 0.289 & 0.359 \\
GRPO-$\lambda$ & 0.317 & 0.299 & 0.299 & 0.366 \\
PPO           & 0.333 & 0.316 & 0.316 & 0.386 \\
\textbf{OSPO-Prop} & \textbf{0.361} & \textbf{0.342} & \textbf{0.342} & \textbf{0.417} \\
\midrule
\multicolumn{5}{l}{\textit{3B fine-tuned}} \\
SFT           & 0.390 & 0.373 & 0.373 & 0.441 \\
GRPO          & 0.385 & 0.368 & 0.368 & 0.434 \\
GRPO-$\lambda$ & 0.393 & 0.376 & 0.376 & 0.443 \\
PPO           & 0.403 & 0.385 & 0.385 & 0.456 \\
\textbf{OSPO-Prop} & \textbf{0.422} & \textbf{0.404} & \textbf{0.404} & \textbf{0.475} \\
\bottomrule
\end{tabular}

\vspace{-5pt}
\caption{H\&M product search performance across 1.5B and 3B scales. Full results shown in \Cref{tab:hnm_small_models_full} (Appendix~\ref{app:additional_results})}
\label{tab:hnm_wrap_credit_assignment}
\vspace{-12pt}
\end{wraptable}   OSPO outperforms all baselines at both scales, while all RL methods improve over the respective base models. At 1.5B, OSPO achieves 
$0.361$ NDCG vs.\ $0.307$ for GRPO, $0.317$ for GRPO-$\lambda$, and 
$0.333$ for PPO—the ordering reflects the credit assignment 
hierarchy: temporal credit via eligibility traces (GRPO-$\lambda$) improves over uniform 
assignment, value-model based credit (PPO) improves further via GAE, yet 
both fall short of OSPO's marginal attribution, which directly 
measures each segment's retrieval contribution without a critic. 
At 3B, GRPO degrades below SFT ($0.385$ vs.\ $0.390$) consistent with similar trends in prior work~\citep{samineni2025rlonlyanalyzingstructural}, while GRPO-$\lambda$ ($0.393$) 
and PPO ($0.403$) partially recover. \ospoProp{} ($0.422$) leads across all metrics at both scales—approaching GPT-4o ($0.433$) at 3B scale—confirming that segment-level marginal attribution is strictly more effective than temporal or value-model-based credit for contextualized product search (\textbf{RQ1}).

\begin{figure*}[t]
  \centering
  \includegraphics[width=\textwidth]{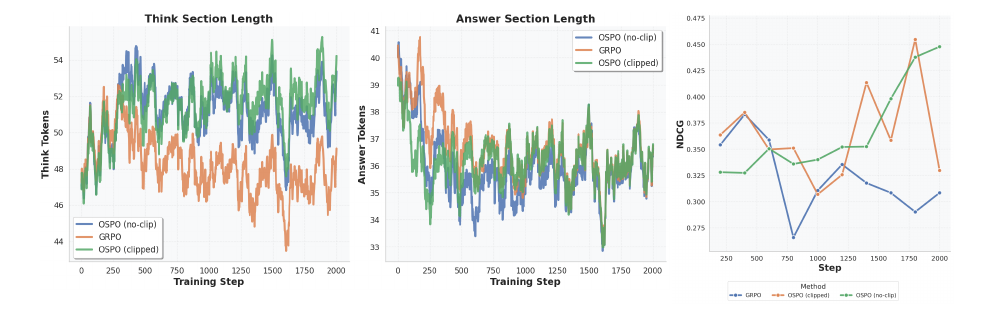}
    \captionsetup{skip=2pt, belowskip=-10pt}  
      \vspace{-5mm}  
 \caption{Left and middle: Chain-of-thought (CoT) lengths within \texttt{<think>} fields and refined query lengths within \texttt{<answer>} fields during RL training of OSPO variants and GRPO on the H\&M product search task. Right: evaluation performance on 500 randomly sampled H\&M test queries, measured every 200 training steps.}
  \label{fig:length_generalization}
\end{figure*}

\begin{wrapfigure}{r}{0.55\textwidth}
  \vspace{-2mm}
  \centering
  \includegraphics[width=0.55\textwidth]{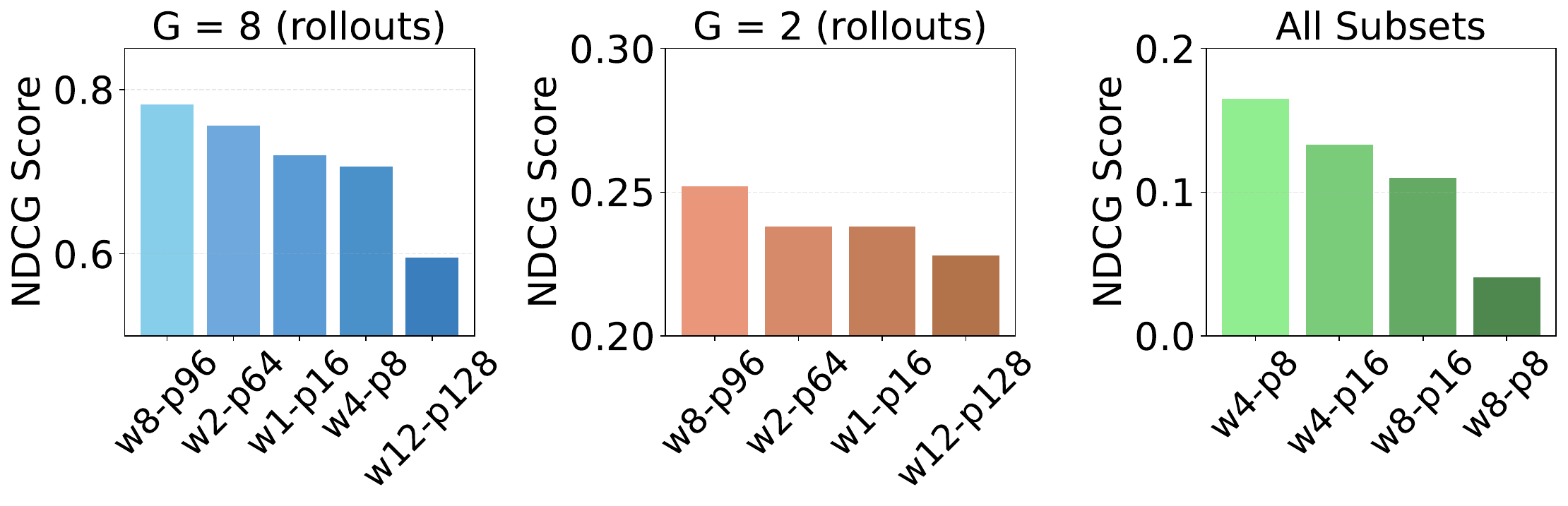}
  \caption{\ospoProp{} ablation results on varying coalition structure on ESCI product search task. \textbf{w} denotes the maximum coalition span ($w_{max}$) and \textbf{p} the number of coalitions sampled ($M$) for Owen value estimation in \ospo{} (\Cref{alg:ospo}).}
  \label{fig:ospo_ablation_plots_esci}
  \vspace{-4mm}
\end{wrapfigure}
\paragraph{Coalition Structure Drives OSPO Performance}
\vspace{-3mm}
OSPO’s performance is governed by the interaction between coalition width, exploration depth, and contiguity. As shown in Fig.~\ref{fig:ospo_ablation_plots_esci} (left), moderately wide coalitions (w8–p96) strike the best balance\footnote{The coalition budget $p{=}96$ upper-bounds the retrieval calls per response; in
typical H\&M responses (N=5--8 phrases, $w_{\max}{=}8$), $|\mathcal{C}|$ 
ranges from 13--40, so the actual coalitions evaluated falls well below the cap of 96 (See Appendix~\ref{ssec:compute_ospo_appendix}).} between local compositional context and sampling diversity (NDCG 0.782). Narrow coalitions overfit to token-level co-occurrences, while overly wide coalitions dilute credit assignment and degrade performance (NDCG 0.595). Exploration is equally critical: reducing generations (rollouts) from eight to two (Fig.~\ref{fig:ospo_ablation_plots_esci}, middle) causes a sharp drop in accuracy, highlighting the need for sufficient Monte Carlo sampling. Finally, removing the contiguity constraint (Fig.~\ref{fig:ospo_ablation_plots_esci}, right) leads to a dramatic collapse in performance (NDCG 0.113), as non-contiguous coalitions disperse credit across unrelated segments. Overall, OSPO performs best with contiguous, moderately wide coalitions and adequate sampling, yielding more stable and generalizable credit assignment.

\vspace{-1mm}
\paragraph{User Profile Quality Analysis}\begin{wraptable}{r}{0.42\textwidth}
\vspace{-4mm}
\centering
\footnotesize
\setlength{\tabcolsep}{5pt}
\renewcommand{\arraystretch}{1.1}
\begin{tabular}{@{}lcc@{}}
\toprule
\textbf{Pair} & \textbf{Ties} & \textbf{WR} \\
\midrule
\ospoRank{} vs.\ GRPO & 42 & 47.6 \small{(43.2)} \\
\ospoRank{} vs.\ DPO  & 47 & 45.8 \small{(43.8)} \\
\ospoRank{} vs.\ SFT  & 40 & \textbf{54.0} \small{(37.2)} \\
\addlinespace[2pt]
\ospoProp{} vs.\ GRPO & 38 & 49.1 \small{(42.5)} \\
\ospoProp{} vs.\ DPO  & 45 & 49.6 \small{(40.5)} \\
\ospoProp{} vs.\ SFT  & 81 & \textbf{53.3} \small{(28.9)} \\
\bottomrule
\end{tabular}
\caption{Pairwise win rates (WR) for user summary generation. 
$\mathrm{WR}_A = (\mathrm{wins}_A + 0.5\,\mathrm{ties})/\mathrm{total}$. 
Parentheses show the baseline's WR.}
\label{tab:ospo_wr_compact}
\vspace{-4mm}
\end{wraptable} Table~\ref{tab:ospo_wr_compact} reports pairwise win rates scored 
by Qwen-3-Nemotron-32B-Reward~\citep{wang2025helpsteer3preferenceopenhumanannotatedpreference}. 
Both \ospoRank{} and \ospoProp{} clearly outperform SFT (54.0\% and 
53.3\% respectively) and maintain competitive performance against 
GRPO and DPO (45.8--49.6\%), with baseline WRs of 40.5--43.8\% 
confirming OSPO wins more decisive head-to-head comparisons in all 
cases. These results are particularly notable given that OSPO and 
GRPO are both initialized directly from the base model without 
expert data warm-up, while SFT and DPO are trained directly on 
expert demonstrations. Unlike product search where 
phrase-level coalitions target the refined query directly, 
summarization employs sentence-level coalitions over an intermediate 
reasoning trace trained using a BT-reward model rather than 
a retriever, probing whether OSPO generalizes 
beyond retrieval-oriented generation (\textbf{RQ2}).  We provide a detailed 
case study on sentence-level Owen values and actual coalitions 
in Tables~\ref{fig:hnm_summary_fullprompt}--\ref{tab:user_summary_examples_last_table}.

\begin{wraptable}{r}{0.42\textwidth}
\centering
\scriptsize
\begin{tabular}{lcccc}
\toprule
\textbf{Method} 
& \multicolumn{2}{c}{\textbf{ESCI}} 
& \multicolumn{2}{c}{\textbf{H\&M}} \\
\cmidrule(lr){2-3} \cmidrule(lr){4-5}
& In & Cross & In & Cross \\
\midrule
GRPO        & 0.418 & 0.301 & 0.379 & 0.184 \\
OSPO (Rank) & 0.485 & 0.368 & 0.399 & 0.193 \\
OSPO (Prop) & 0.522 & 0.394 & 0.436 & 0.190 \\
\bottomrule
\end{tabular}
\caption{NDCG under retriever shift.}
\label{tab:cross_retriever_robustness_trimmed}
 
\end{wraptable}
 
\vspace{-2.5mm}

\paragraph{OSPO Demonstrates Superior Generalization}
Fig.~\ref{fig:length_generalization} (right) reveals that \ospoProp{} exhibits stable improvement on test data throughout training (reaching NDCG $\approx$ 0.45), while GRPO shows catastrophic performance collapse around step 400-750. This generalization gap manifests despite GRPO achieving higher training rewards (0.7603 vs. 0.6975 for OSPO), highlighting a “classic reward-hacking as overfitting” scenario. Two behavioral markers explain this pattern: First, GRPO~\citep{shao2024deepseekmath} \textit{systematically reduces thinking tokens} (roughly around 300th step mark) compared to \ospo{} variants precisely when test performance degrades (see Fig.~\ref{fig:length_generalization} (left)), while \ospo{} maintains consistently longer reasoning traces. Second, GRPO's reward variance drops sharply toward 0.30 by training end, compared to \ospo{}'s higher variance (0.34-0.36), indicating GRPO converges to consistent but limited solutions that exploit training patterns without transferring to unseen data. Additionally, \Cref{tab:cross_retriever_robustness_trimmed} reports a cross-retriever evaluation designed to test generalization across embedding spaces (\textbf{RQ3}). Models are trained using one retriever (ESCI with \textsc{all-mpnet-base-v2} or H\&M with \textsc{SIMCSE-Large}) and evaluated under a retriever shift using the alternate embedding model. Across both transfer directions, \ospo{} maintains strong ranking performance, while GRPO degrades substantially, indicating that Owen value-based attributions learn a more transferable and retriever-agnostic alignment signal. Further analysis and evidence of generalization to a non-retrieval open-ended generation task is provided in 
\Cref{tab:cross_retriever_results} and \Cref{fig:imdb_training_dynamics} 
in Appendix~\ref{app:additional_results}.

\vspace{-4mm}
\section{Conclusion and Future Work}
\vspace{-2mm}
We introduced \textbf{Owen–Shapley Policy Optimization (\ospo{})}, a reinforcement learning framework that directly tackles the credit assignment problem in language-driven recommendation systems. By decomposing sequence-level rewards into segment-level attributions through cooperative game theory, \textsc{OSPO} enables precise gradient updates without relying on auxiliary value networks. Across product search and user summarization tasks, \textsc{OSPO} consistently outperforms both online and offline baselines—including models an order of magnitude larger—while approaching frontier-level performance at just 3B scale. This efficiency arises from its attribution-driven optimization, which focuses updates on the sequence segments that meaningfully influence task outcomes. Even with unseen retrievers, \textsc{OSPO} preserves robust performance, showing that its alignment signal transfers across feedback models and representation spaces. Ablation studies further highlight the importance of its coalition structure: moderate-width, contiguous, and sufficiently sampled coalitions act as an intrinsic regularizer that stabilizes learning and improves generalization. In summary, \textsc{OSPO} learns not by memorizing feedback patterns but by leveraging \textit{principled coalition-based attributions that promote stable, transferable, and interpretable behaviors}. Future work could explore such “coalition-based” credits during multiturn interactions with user simulators~\citep{wucollabllm} or at various levels of abstractions—e.g., one could think of semantically similar sub-groups in GRPO across batches for meta-learning~\citep{zhang2025grpoleaddifficultyawarereinforcementlearning} or to leverage Owen values in multi-agent settings~\citep{nath2026craft} for the challenging problem of assigning agent-specific credits for optimal coordination.


\newpage
 
\bibliography{combined_deduped.bib}

\newpage
\appendix
\onecolumn
 
\section{Appendices}

\begin{table}[h]
\centering
\small
\setlength{\tabcolsep}{8pt}
\renewcommand{\arraystretch}{1.2}
\begin{tabular}{lp{9cm}}
\toprule
\textbf{Section} & \textbf{Contents} \\
\midrule
\Cref{app:proofs} & Extended theoretical derivations supporting the OSPO objective \\
\Cref{sec:redistribution} & Comparison to GRPO length-invariant credit redistribution \\
\midrule
\Cref{app:expert_data_generation_prompts} & Prompts and procedures for H\&M synthetic data generation \\
\Cref{app:baselines} & Baseline configurations and prompt examples \\
\Cref{app:additional_results} & Supplementary experimental results \\
\midrule
\Cref{ssec:coalition_formation} & Coalition formation implementation, complexity, and sampling strategies \\
\Cref{ssec:compute_ospo_appendix} & Owen-Shapley attribution validity, wall-clock breakdown vs.\ GRPO \\
\Cref{ssec:coalition_ablations} & Comprehensive ablation studies on coalition design choices \\
\Cref{ssec:coalition_learning_dynamics} & Coalition width and learning dynamics analysis \\
\Cref{ssec:coalition_learning_dynamics_smaller} & Coalition dynamics robustness under reduced model capacity \\
\bottomrule
\end{tabular}
\caption*{Appendix organization.}
\end{table}
\vspace{-3mm}

\vspace{-3mm}

\paragraph{Limitations.}
\textsc{OSPO} incurs additional computational overhead compared to standard policy-gradient methods, since estimating Owen values for advantage redistribution adds extra computation. While FAISS-based retrieval feedback keeps this cost manageable for product search, it becomes more noticeable when using slower reward models, as in user profile summarization. Still, the trade-off is favorable for precision-critical applications where \textsc{OSPO} achieves higher performance and efficiency than much larger models. Several practical challenges remain. Performance is sensitive to coalition width and sampling depth, making hyperparameter tuning crucial but challenging to the practitioner. Additionally, Monte Carlo approximations of Owen values during coalition sampling may underrepresent long-range dependencies between non-adjacent segments, and model stability can vary with feedback quality across domains. Finally, our evaluation spans only two recommendation settings and single-turn RL; extending \textsc{OSPO} to multi-turn, conversational, and multimodal environments remains promising future work.
\vspace{-3mm}
\paragraph{Reproducibility Statement}
We are committed to full reproducibility. Our implementation for our main experimental pipeline is based on the TRL 
GRPO trainer\footnote{\url{https://huggingface.co/docs/trl/main/en/grpo_trainer}} 
with custom modifications for OSPO.

\subsection{More Details on OSPO Objectives}
\label{app:proofs}
\begin{lemma}[Length-Invariant Owen Redistribution]
\label{lemma:length_invariant}
Let $y^{(g)}$ be a generated sequence of length $T$ from a Monte Carlo rollout with reward $R^{(g)}$. Following the REINFORCE algorithm~\citep{williams92,kool2019buy, shao2024deepseekmath} of sampling rollouts without replacement with baseline subtraction, define the sequence-level advantage as:
\begin{equation}
\hat{A}^{(g)} = R^{(g)} - \bar{R}, \quad \text{where } \bar{R} = \frac{1}{G}\sum_{g'=1}^G R^{(g')},
\end{equation}
is the mean reward across $G$ sampled rollouts. Let $\{\phi_t^{(g)}\}_{t=1}^T$ denote Owen values satisfying the efficiency axiom $\sum_{t=1}^T \phi_t^{(g)} = R^{(g)}$~\citep{cheng2014owen,owen2014sobol}. Define the normalized Owen values as $\tilde{\phi}_t^{(g)} = \phi_t^{(g)} / \sum_{t'=1}^T \phi_{t'}^{(g)}$, and the redistributed token advantages as:
\begin{equation}
A_t^{(g)} = T \cdot \tilde{\phi}_t^{(g)} \cdot \hat{A}^{(g)}.
\end{equation}
Then the average token advantage equals the original sequence-level advantage, independent of sequence length:
\begin{equation}
\frac{1}{T}\sum_{t=1}^T A_t^{(g)} = \hat{A}^{(g)}.
\end{equation}
\end{lemma}

\begin{proof}
To prove this, starting from the definition of the redistributed token advantages, we compute the average:
\begin{align}
\frac{1}{T}\sum_{t=1}^T A_t^{(g)} 
&= \frac{1}{T}\sum_{t=1}^T \left( T \cdot \tilde{\phi}_t^{(g)} \cdot \hat{A}^{(g)} \right) \\
&= \frac{T}{T} \cdot \hat{A}^{(g)} \sum_{t=1}^T \tilde{\phi}_t^{(g)} \\
&= \hat{A}^{(g)} \sum_{t=1}^T \tilde{\phi}_t^{(g)}.
\end{align}

Now we evaluate the sum of normalized Owen values. By the definition of $\tilde{\phi}_t^{(g)}$:
\begin{align}
\sum_{t=1}^T \tilde{\phi}_t^{(g)} 
&= \sum_{t=1}^T \frac{\phi_t^{(g)}}{\sum_{t'=1}^T \phi_{t'}^{(g)}} \\
&= \frac{1}{\sum_{t'=1}^T \phi_{t'}^{(g)}} \sum_{t=1}^T \phi_t^{(g)} \\
&= \frac{\sum_{t=1}^T \phi_t^{(g)}}{\sum_{t'=1}^T \phi_{t'}^{(g)}} \\
&= \frac{\sum_{t=1}^T \phi_t^{(g)}}{\sum_{t=1}^T \phi_t^{(g)}} \\
&= 1.
\end{align}

Substituting this result back into equation (3), we obtain:
\begin{equation}
\frac{1}{T}\sum_{t=1}^T A_t^{(g)} = \hat{A}^{(g)} \cdot 1 = \hat{A}^{(g)}.
\end{equation}

This equality holds for any sequence length $T > 0$, proving that the redistribution scheme is length-invariant.
\end{proof}

\begin{remark}
Lemma~\ref{lemma:length_invariant} guarantees three critical properties for OSPO's credit assignment:
\begin{enumerate}
    \item \textbf{Scale Preservation:} The average token advantage matches the original sequence-level advantage $\hat{A}^{(g)} = R^{(g)} - \bar{R}$, ensuring gradient magnitudes in \Cref{eq:ospo_obj} remain comparable across sequences of different lengths.
    
    \item \textbf{Length Invariance:} A sequence's total gradient contribution depends on its advantage $\hat{A}^{(g)}$ and Owen-based credit distribution $\{\tilde{\phi}_t^{(g)}\}$, not its verbosity. This eliminates the length bias where longer sequences (e.g., detailed user summaries with $T > 100$) would otherwise receive disproportionately small gradient signals when Owen values are naively redistributed.
    
    \item \textbf{Credit Prioritization:} High-value tokens (e.g., product-specific attributes) receive proportionally more advantage than low-value tokens (e.g., filler words or inconsistent reasoning), while maintaining consistency with the per-token policy gradient formulation. Specifically, if $\tilde{\phi}_i > \tilde{\phi}_j$, then $A_i^{(g)} > A_j^{(g)}$, preserving the ordinal ranking of token contributions.
\end{enumerate}
\end{remark}

\begin{remark}[Connection to the Efficiency Axiom]
The normalization step $\tilde{\phi}_t = \phi_t / \sum_{t'} \phi_{t'}$ is essential because Owen values satisfy the efficiency axiom: $\sum_{t=1}^T \phi_t = R^{(g)}$ (the terminal reward). Without the length multiplication factor $T$, a naive proportional redistribution $\tilde{A}_t = \hat{A}^{(g)} \cdot \tilde{\phi}_t$ would yield an average advantage of:
\begin{equation}
\frac{1}{T}\sum_{t=1}^T \tilde{A}_t = \frac{\hat{A}^{(g)}}{T} \sum_{t=1}^T \tilde{\phi}_t = \frac{\hat{A}^{(g)}}{T},
\end{equation}
introducing a $1/T$ dilution effect where longer sequences receive significantly weaker gradient signals. The multiplication by $T$ precisely counteracts this dilution, ensuring that gradient magnitude depends on the quality of the generation (via $\hat{A}^{(g)}$) rather than its length.  Algorithm~\ref{alg:ospo} provides the detailed OSPO training  procedure. 
\end{remark} 

\begin{algorithm}[t]
\scriptsize
\caption{Owen-Shapley Policy Optimization (OSPO)}
\label{alg:ospo}
\begin{algorithmic}[1]
\Require Policy $\pi_\theta$, dataset $\mathcal{D}$, reward $r$, segmentation $\text{Segment}(\cdot)$,  coalition width $w_{\max}$, samples $M$, group size $G$
\For{$x \sim \mathcal{D}$}
    \State Sample $G$ responses $\{y^{(g)}\}_{g=1}^G \sim \pi_{\theta_{\text{old}}}(\cdot \mid x)$
    \For{$g = 1, \ldots, G$}
        \State Segment $\mathcal{A}^{(g)} = \text{Segment}(y^{(g)})$ and compute $r^{(g)} = r(x, y^{(g)})$

       \State Sample $M$ coalitions $\{S_m\}_{m=1}^M$ with $S_m \subseteq \{1,\ldots,N\}$, $|S_m| \leq w_{\max}$
        \State Compute Owen values: $\phi_j = \frac{1}{|C_j|} \sum_{S \in C_j} [v(S \cup \{j\}) - v(S)]$
        \State Map to tokens: $\boldsymbol{\phi}^{(g)} = (W^{(g)})^\top \boldsymbol{\phi}^{(g)}_{\text{seg}}$
        \State Normalize: $\tilde{\phi}_t^{(g)} = \phi_t^{(g)} / \sum_{t'} \phi_{t'}^{(g)}$
        \State Compute advantage: $\hat{A}^{(g)} = (r^{(g)} - \bar{r}) / \sigma_r$
        \State Redistribute: $A_t^{(g)} = T \cdot \tilde{\phi}_t^{(g)} \cdot \hat{A}^{(g)}$
        \State Compute ratio: $\rho_t^{(g)} = \pi_\theta(y_t^{(g)} \mid x, y_{<t}^{(g)}) / \pi_{\theta_{\text{old}}}(y_t^{(g)} \mid x, y_{<t}^{(g)})$
    \EndFor
    \State Update: $\mathcal{L}_{\text{OSPO}} = \frac{1}{GT}\sum_{g,t} \min(\rho_t^{(g)} A_t^{(g)}, \text{clip}_\epsilon(\rho_t^{(g)}) A_t^{(g)})$
\EndFor
 
\end{algorithmic}
\end{algorithm}

\subsection{Length-Invariant Credit Redistribution for GRPO}
\label{sec:redistribution}

Group Relative Policy Optimization (GRPO) \citep{shao2024deepseekmath} computes sequence-level advantages by normalizing rewards within a sampled group, but assigns this advantage uniformly to all tokens in the sequence. This uniform assignment creates a fundamental tension: when combined with the standard per-token gradient normalization $\frac{1}{|y|}$, GRPO exhibits length-dependent bias where positive advantages produce stronger gradients for shorter sequences, while negative advantages under-penalize verbose failures \citep{liu2025understandingr1zeroliketrainingcritical}. This bias is particularly problematic in recommendation tasks where response length varies significantly—from terse product queries to detailed user preference summaries.

Our approach leverages Owen-Shapley values to redistribute the sequence-level GRPO advantage to individual tokens while explicitly eliminating length bias. Given a sequence $y$ with GRPO advantage $\hat{A}$ and Owen values $\{\Phi_t\}_{t=1}^{|y|}$ (obtained via the efficient approximation in~\Cref{ssec:owen-shap-values}), we face a redistribution challenge. A naive proportional approach would compute token advantages as $\tilde{A}_t = \hat{A} \cdot \frac{\Phi_t}{\sum_{t'} \Phi_{t'}}$. However, since Owen values satisfy the efficiency axiom ($\sum_t \Phi_t = R$, where $R$ is the terminal reward), this denominator equals the reward $R$, causing per-token advantages to scale as $\hat{A}/R$. For sequences with correlated advantages and rewards, this creates the same length bias as vanilla GRPO, where the average token advantage $\frac{1}{|y|}\sum_t \tilde{A}_t \propto \frac{1}{|y|}$.

To eliminate this dependency, we propose \textbf{two }redistribution strategies that multiply the proportional allocation by sequence length $|y|$, ensuring the average token advantage equals the original GRPO advantage regardless of sequence length. This length normalization is critical: it decouples gradient magnitude from verbosity while preserving Owen-based prioritization of important tokens.

\paragraph{Proportional Owen Redistribution} Our first approach uses Owen values directly as importance weights:
\begin{equation}
\tilde{A}_t^{\text{prop}} = \hat{A} \cdot \frac{\Phi_t}{\sum_{t'} \Phi_{t'}} \cdot |y|
\end{equation}
This strategy assigns credit proportional to each token's marginal contribution, as quantified by its Owen value. Tokens in high-value\footnote{We assume that we have access to a reasonably efficient and realistic retrieval system that is more likely pull up relevant articles for product-specific attributes, while lowering those of filler words.} phrases (e.g., ``noise cancelling'' in a headphone query) receive proportionally more advantage than tokens in low-value phrases (e.g., filler words). The length multiplication ensures that $\frac{1}{|y|}\sum_t \tilde{A}_t^{\text{prop}} = \hat{A}$, making the average token advantage independent of sequence length. Our experiments in product search tasks suggests that this redistribution is particularly effective when Owen values accurately reflect token importance—typically in short sequences with clear, attributable contributions.

\paragraph{Rank-Based Owen Redistribution} Our second approach converts Owen values to ordinal ranks before redistribution:
\begin{equation}
\text{rank}_t = \text{argsort}(\{\Phi_{t'}\}), \quad \tilde{A}_t^{\text{rank}} = \hat{A} \cdot \frac{\text{rank}_t}{\sum_{t'} \text{rank}_{t'}} \cdot |y|
\end{equation}
This rank transformation acts as a form of regularization on the Owen-based credit assignment. While proportional redistribution can concentrate credit heavily on a few high-value tokens (e.g., one token receiving 80\% of the advantage), rank-based redistribution spreads credit more uniformly according to relative importance order. For a sequence of length $N$, the top-ranked token receives weight $\frac{N-1}{\sum_{k=0}^{N-1}k} = \frac{2}{N}$ instead of a potentially much larger proportion. This uniform spreading is advantageous in two scenarios: (1) when Owen value estimates are noisy due to complex token interactions in long sequences, the ordinal transformation filters magnitude errors while preserving relative importance; (2) when task success requires multiple complementary components (e.g., multiple attributes in a product query), rank-based redistribution prevents over-concentration on a single dominant feature and encourages the model to improve all relevant tokens.

Both strategies maintain the theoretical guarantees of potential-based reward shaping \citep{ng1999policy}. The choice between proportional and rank-based redistribution depends on the task structure: proportional excels when token contributions are clearly separable and accurately measurable, while rank-based provides robustness when contributions are compositional and interdependent.

\subsection{Expert Data Generation}
\label{app:expert_data_generation_prompts}

Unlike ESCI~\citep{reddy2022shopping}, the H\&M Fashion Recommendation dataset~\citep{h-and-m-personalized-fashion-recommendations} lacks explicit customer queries—crucial for evaluating search-oriented LLMs. To ensure a fair comparison of \textsc{OSPO} with SFT, preference-tuning, and RL baselines (e.g., GRPO), we conduct an \textbf{expert data generation phase} using \texttt{Claude Sonnet 3.0}. The expert model synthesizes missing annotations from sampled customer profiles drawn from rich transactional histories. Specifically, it produces \textbf{(i)} next-purchase predictions and preference-grounded summaries for the \textit{user summarization} task, and \textbf{(ii)} naturalistic, history-conditioned queries for the \textit{contextualized product search} task (see \Cref{fig:hnm_prompt_example}). The full pipeline proceeds as follows.

\paragraph{Customer Data Preparation}
We construct a sequential recommendation dataset by segmenting each customer’s purchase history into temporally ordered windows. A customer is retained if they have at least five prior and four subsequent purchases, providing adequate temporal context while including light users. To maintain computational efficiency, we subsample a fixed number of eligible customers, yielding a balanced and diverse training cohort.

\paragraph{Expert Data Generation}
Following prior work on RL-based sequential recommendation~\citep{wu2025rlpfreinforcementlearningprediction}, we generate \textit{expert user summaries} that compress behavioral histories into concise representations. Because no ground truth exists for such summaries or reasoning trails, the expert LLM produces Chain-of-Thought (CoT) explanations grounded in observed behavior—capturing both short- and long-term preferences. Unlike~\citet{wu2025rlpfreinforcementlearningprediction}, which predicts only the next item, our expert forecasts entire future sequences, pruning implausible candidates while reasoning over temporal dependencies. These CoT traces form structured summaries encoding customer intent and preference dynamics.

\paragraph{SFT and DPO Trajectories for User Profile Summarization}
From the generated expert data, we construct SFT and DPO~\citep{rafailov2024direct} training sets based on the expert’s performance in sequential purchase prediction, where future items are ranked in temporal order. To ensure data quality, we retain only samples where the expert’s NDCG exceeds a threshold $T=0.5$, tuned empirically. These high-performing traces constitute the SFT data. For DPO, we obtain contrastive pairs by sampling two reasoning traces per example—keeping the higher-NDCG trace as the “winning” response and the lower one as the “losing” response. This consistent expert-driven generation enables fair comparative analysis across SFT, DPO, and RL-based methods trained under identical supervision signals.

\paragraph{Contextualized Query Generation}
For the contextualized product search task, we prompt the same expert model to generate realistic, user-grounded fashion queries conditioned on each customer’s purchase history and item context. Using the role-based prompt in \Cref{fig:hnm_prompt_example}, the expert produces diverse, natural-language queries that mimic how real users express shopping intent—balancing explicit style and brand references with implicit preferences, seasonal cues, and underspecified needs. Each generated query is paired with a detailed item rationale, enabling the resulting dataset to capture nuanced intent representations crucial for training and evaluating OSPO on fashion-specific retrieval tasks.

\paragraph{Comparison of Expert-generated H\&M Queries with Amazon C4 and ESCI} For additional validation, we compare the linguistic characteristics of H\&M expert-generated queries above against ESCI~\citep{reddy2022shopping}---a dataset of real Amazon user query logs covering 3M items across 33 categories---and Amazon C4~\citep{hou2024bridging}, a benchmark of “complex” customer queries grounded in purchase history. As shown in Table~\ref{tab:linguistic}, H\&M queries closely mirror the C4 paradigm across all dimensions: high first-person framing (96.7\% vs.\ 99.3\%), substantial comparison and negation usage (37.7\% vs.\ 14.8\%), and strong occasion-driven intent (56.2\% vs.\ 10.5\%). Lexical diversity (MATTR, window=50) is also comparable, confirming that H\&M queries exhibit the naturalistic variability of real shopping language rather than templated generation.

\begin{table}[h]
\centering
\small
 
\begin{tabular}{lccc}
\toprule
\textbf{Metric} & \textbf{ESCI} & \textbf{Amazon C4} & \textbf{H\&M (ours)} \\
\midrule
First-person      & 0.4\%  & 99.3\% & 96.7\% \\
Comparison/negation & 0.1\% & 14.8\% & 37.7\% \\
Occasion-driven   & 0.9\%  & 10.5\% & 56.2\% \\
MATTR (window=50) & 0.998  & 0.864  & 0.674  \\
\bottomrule
\end{tabular}
\caption{Linguistic comparison of query distributions across ESCI, Amazon C4, and H\&M expert-generated queries.}
\label{tab:linguistic}
\end{table}

\paragraph{Independent Validation with GPT-4o Judge} To further validate expert-generated query quality (originally generated using the prompt shown in \Cref{fig:hnm_prompt_example}), we conduct an (\Cref{fig:hnm_expert_validation_gpt}) independent GPT-4o judge evaluation  on 600 randomly sampled query-item pairs (200 per seed, $n{=}3$ seeds) from the evaluation set of H\&M queries used in our main experiments (\Cref{sec:experimental_setup}). Each query is evaluated against the \textbf{ground-truth purchased item} and its associated item description, providing the judge with the \textbf{same purchase signal} used to generate the query, enabling assessment of whether the expert model produces queries that genuinely reflect real purchase intent rather than generic shopping language. The judge scores each query on \textbf{three} dimensions: \textit{relevance} (does it plausibly surface the target item?), \textit{naturalness} (does it reflect how a real fashion customer expresses shopping needs?), and \textit{groundedness} (does it reflect specific, purchase-grounded intent?), each rated on a 1--5 Likert scale with ground-truth purchase context provided. As reported in Table~\ref{tab:gpt4o_val}, expert-generated queries are judged highly natural ($4.96/5$, $\sigma{=}0.19$) and grounded ($4.75/5$, $\sigma{=}0.63$). Relevance variance ($\sigma{=}1.12$) reflects the inherent difficulty of anticipating a specific target item from purchase history alone; the mean score ($4.46/5$) nonetheless confirms strong alignment between query intent and retrieval target.

\begin{table}[h]
\centering
\small
 
\begin{tabular}{lcc}
\toprule
\textbf{Dimension} & \textbf{Mean} & \textbf{Std} \\
\midrule
Relevance    & 4.46 & 1.12 \\
Naturalness  & 4.96 & 0.19 \\
Groundedness & 4.75 & 0.63 \\
\bottomrule
\end{tabular}
\caption{GPT-4o independent judge scores for H\&M expert-generated queries across 600 query-item pairs (200 per seed, $n{=}3$). Prompt used for this evaluation is shown in \Cref{fig:hnm_expert_validation_gpt}.}
\label{tab:gpt4o_val}
\end{table}

Together, these results confirm that H\&M expert queries exhibit the linguistic register and intent structure of real contextualized shopping queries, thereby validating our experimental design for contextualized product search in \Cref{sec:experimental_setup}.

\begin{table}[t]
\centering
\setlength{\tabcolsep}{3pt}      
\renewcommand{\arraystretch}{0.92}
\footnotesize
\begin{tabular}{lccc}            
\toprule
\textbf{Metric} & \textbf{GPT-4o} & \textbf{Claude-Sonnet Run 1} & \textbf{Claude-Sonnet Run 2} \\
\midrule
Hit@1                & 0.2622 & 0.2719 & 0.2715 \\
MRR                  & 0.4995 & 0.5021 & 0.5056 \\
NDCG                 & 0.7413 & 0.7394 & 0.7412 \\
Spearman             & 0.4747 & 0.4723 & 0.4708 \\
Hit@3                & 0.6615 & 0.6593 & 0.6729 \\
MRR@3                & 0.4393 & 0.4406 & 0.4458 \\
NDCG@3               & 0.5970 & 0.5951 & 0.5990 \\
Hit@5                & 0.8095 & 0.8218 & 0.8305 \\
MRR@5                & 0.4730 & 0.4773 & 0.4817 \\
NDCG@5               & 0.6683 & 0.6680 & 0.6719 \\
Hit@7                & 0.9179 & 0.9272 & 0.9377 \\
MRR@7                & 0.4901 & 0.4938 & 0.4985 \\
NDCG@7               & 0.7125 & 0.7118 & 0.7166 \\
Hit@9                & 0.9835 & 0.9834 & 0.9843 \\
MRR@9                & 0.4979 & 0.5005 & 0.5041 \\
NDCG@9               & 0.7359 & 0.7342 & 0.7364 \\
\bottomrule
\end{tabular}
\caption{Overall averages across 5700 samples from the  H\&M Fashion Recommendations data~\citep{h-and-m-personalized-fashion-recommendations} for expert data generation on user summary generation on H\&M Fashion Recommendations data. Expert models are instructed to reason about the user's preferences based on their purchase history, before predicting the next purchases items, conditioned on this reasoning trace. 
  }
\label{tab:gpt_claude_combined}
\end{table}

 \begin{figure}[t]
\centering
\small
\begin{tcolorbox}[
colback=white, colframe=black!70,
    title={\small\textbf{GPT-4o Judge Prompt for H\&M Query Validation}},
    fonttitle=\bfseries,
    breakable
]
\ttfamily\small
You are evaluating the quality of a customer search query in a fashion retrieval system.

\medskip
\textbf{Given:}
\begin{itemize}[leftmargin=1.5em, topsep=2pt, itemsep=1pt]
    \item \textbf{Customer query:} a search expression the customer used
    \item \textbf{Target item:} the item they actually purchased
    \item \textbf{Item context:} a description of why the item matches the query
\end{itemize}

\medskip
Rate the query on three dimensions (1--5 each):

\medskip
\textbf{1. RELEVANCE:} Does the query plausibly lead to this target item? (5 = directly describes the item, 1 = completely unrelated)

\medskip
\textbf{2. NATURALNESS:} Does the query reflect how a real fashion customer thinks and expresses shopping needs? (5 = genuine customer intent with personal context and natural conversational tone, 1 = generic or reads like a product description)

\medskip
\textbf{3. GROUNDEDNESS:} Does the query reflect real shopping intent with specific context? (5 = rich context with specific needs and personal references, 1 = vague or generic)
\medskip
\textbf{Query:} \{query\}\\
\textbf{Target item:} \{item\_context\}

\medskip
Respond ONLY with valid JSON:\\
\{"relevance": \textlangle1--5\textrangle, "naturalness": \textlangle1--5\textrangle, "groundedness": \textlangle1--5\textrangle, "reasoning": "\textlangle one sentence\textrangle"\}
\end{tcolorbox}
\caption{Prompt used for independently validating expert-model generated  H\&M contextualized queries using the GPT-4o model as LLM Judge.}
\label{fig:hnm_expert_validation_gpt}
\end{figure}

\paragraph{Prompt Construction for User Profile Generation and Expert Validation}
To instantiate this task for large language models, we construct prompts that embed
both the historical prefix and the candidate future set. Each prompt begins with an
explicit task description, e.g., \emph{``You are an expert sequential recommendation
system. Given the customer's purchase history and buying behavior, predict the exact
order in which the customer will buy the next few products.''} This instruction
enforces that the output must be a \textbf{complete ordering} rather than isolated
recommendations, capturing temporal and long-range dependencies in user behavior.
The history is serialized in natural language with timestamps, prices, and product
descriptions (e.g., \emph{``Transaction 3: 2018-12-20 | Price: \$0.01 -- Russell
Tube -- Scarf in Beige \ldots''}), providing temporal and semantic context, similar
to~\citep{wu2025rlpfreinforcementlearningprediction}. Candidate items are presented
as a numbered list with article IDs and textual metadata, and the model is asked to
\emph{``rank all candidate products in chronological order, with the leftmost article
ID being purchased first.''} To encourage structured reasoning, we include a
chain-of-thought directive: \emph{``Please think step by step about the purchase
sequence and output your reasoning within \texttt{<think>} and \texttt{</think>}
tokens.''}  The model then outputs a ranked list of purchase items directly aligned
with the ground-truth chronological sequence, enabling automatic and verifiable
evaluation. See \Cref{fig:hnm_summary_fullprompt} for the full prompt with outputs. 

To validate the expert model used for generating user profile summaries, we evaluate
its sequential next-item prediction performance on the full set of customer contexts
used for training and evaluation (5{,}700 prompt contexts; \Cref{sec:experimental_setup}).
The rationale is that accurate next-item prediction requires coherent reasoning about
latent user preferences---the same reasoning that conditions the quality of the
generated profile summary. An expert that ranks held-out future purchases highly must
have internalized a faithful representation of the user's preference profile; poor
profile summaries would produce random or inconsistent rankings. \Cref{tab:gpt_claude_combined}
reports results across GPT-4o and \textbf{two} independent runs of Claude Sonnet~3.0. All
models achieve consistent performance ($\text{NDCG} \approx 0.74$, Spearman
$\approx 0.47$), with negligible variance across independent runs (NDCG range
$< 0.003$), confirming that the expert reliably infers user preferences and that the
generated profile summaries constitute stable, high-quality supervision signal for
downstream alignment.

\begin{figure}[t]
  \centering
  \begin{minipage}{0.95\linewidth}
    \begin{quote}
\textbf{System Prompt:} You are an expert in query expansion for product retrieval. Your task: enrich the customer\'s query with product attributes and context that improve dense retrieval accuracy. \textbf{Important: RETAIN the original query terms and ADD relevant product details.} Your expansion should be a natural phrase combining the original query with product attributes.

\textbf{Strategy:}
\begin{enumerate}
    \item Keep all original query terms
    \item Add product category if not explicit (e.g., "laptop" for vague queries)
    \item Include key distinguishing attributes: brand, model, specs, materials, use-case
\end{enumerate}

Your expansion should be a natural phrase combining the original query with product attributes, NOT a first-person request. Format your response as follows:
First, analyze the query in \texttt{<think>} tags: What product category is this? What key attributes would help distinguish the right item from similar products? What buyer intent or use-case is implied?
\newline
Then, output your enriched search query in \texttt{<answer>} tags as a natural phrase combining the original query with product attributes.
\newline

\textbf{User Input:} \texttt{<INPUT QUERY> \textcolor{purple}{(smith and wesson 642 holster)}}
    \end{quote}
  \end{minipage}
\caption{Full prompt used for training and evaluation on the ESCI product search dataset~\citep{reddy2022shopping}. Following~\citep{lin2025rec}, we simplify the format by directly requesting the refined query within \texttt{<answer>} fields instead of JSON-style \texttt{<query>} tags. Only the text within \texttt{<answer>} tags is used for Owen value computations in \textsc{OSPO}, while the \texttt{<think>} section supports intermediate CoT reasoning. The example query shown in \textcolor{purple}{purple} is drawn from the dataset.}

  \label{fig:esci_search_prompt}
\end{figure}

\begin{figure}[t]
  \centering
  \begin{minipage}{0.95\linewidth}
    \begin{quote}
\textbf{System Prompt:} You are a query refinement specialist. Your task is to transform customer queries into product-focused descriptions while retaining all essential information from the original query. 

\textbf{Write your analysis in \texttt{<think>} tags by first identifying the customer's core need}, then noting the specific attributes they mentioned (fit, color, style, occasion, material), and finally checking if purchase history reveals relevant patterns. 

Then \textbf{output your refined query in \texttt{<answer>} tags as 2--3 natural sentences} written as a product description. Start directly with product features. Focus on materials, style, and how it addresses the need. Do \textbf{not} start with phrases like ``The customer needs'' or invent details not present in the query or purchase history. Add coordination with purchase history items only when a clear pattern exists.

\textbf{User Input:} \\
\texttt{<INPUT QUERY> \textcolor{purple}{(I need a new black top, maybe something with embroidery or smocking. Not too tight or low-cut, I prefer a relaxed fit.)}} \\[4pt]
\textbf{Purchase History:} \\
\small
\texttt{Pirat leggings -- Leggings/Tights in Black \\ 
DIV Swift jumper -- Sweater in Dark Blue \\ 
CHERRY JRSY SKIRT -- Skirt in Black \\ 
Control Lt 40 den 2p Tights -- Underwear Tights in Black \\ 
Jiminy -- Sweater in Off White \\ 
Bono NW slim denim -- Trousers in Light Blue \\ 
Wille boot -- Boots in Black \\ 
Bono NW slim denim -- Trousers in Dark Grey \\ 
Luna skinny 5 pkt -- Trousers in Light Pink \\ 
Skirt Pencil Stretch Slit -- Skirt in Black}
    \end{quote}
  \end{minipage}
\caption{Full prompt used for training and evaluation on the H\&M Fashion dataset. Unlike ESCI, this setup grounds the LLM’s instructions in the user’s \textit{purchase history} for contextualized query refinement. Only the text within \texttt{<answer>} tags is used for Owen value computations in \textsc{OSPO}, while the \texttt{<think>} section supports intermediate CoT reasoning.  The example query shown in \textcolor{purple}{purple} is drawn from the dataset.}

  \label{fig:hnm_fashion_prompt}
\end{figure}

\begin{figure}[t]
\centering
\small
\begin{tcolorbox}[width=\textwidth, colback=white, colframe=black!70, title={\sc Contextualized Query Generation Prompt}]

\vspace{0.5em}
\textbf{User:} You are a customer behavior specialist. Generate realistic customer queries that reflect how people actually search for fashion items, based on their purchase history and personal shopping context\\
\medskip
\textbf{[Customer Profile]}\\
\texttt{Customer ID: \{customer\_id\}}\\
\texttt{Transaction 1: \{date\} | Price: \$X -- \{item\_name\} -- 
\{product\_type\} in \{color\}}\\
\texttt{Transaction 2: ...}

\medskip
\textbf{[Candidate Items]}\\
\texttt{1. [ITEM\_ID]: \{item\_name\} -- \{product\_type\} in 
\{color\} \{price: \$X\}}\\
\texttt{2. ...}

\vspace{0.8em}
\textbf{Query Requirements:}
\begin{itemize}
  \item Create five diverse queries (1--3 sentences each) targeting different candidate items.
  \item Vary types: explicit brand/style references, implicit descriptions, negations, underspecified needs.
  \item Include customer context: reference past purchases by name/style, seasonal timing, wardrobe gaps, price expectations.
  \item Use natural fashion language, not technical product descriptions.
  \item Include information gaps (missing colors, wrong categories, vague descriptions).
\end{itemize}

\textbf{Item Context Requirements:}
\begin{itemize}
  \item Link customer needs to the specific product (why this item).
  \item Comparative references to owned items (e.g., ``unlike my Julia denim trousers'').
  \item Underlying motivations (work confidence, wardrobe refresh, versatility, seasonal transition).
  \item Mix natural language \emph{and} precise product terminology.
  \item State both what the customer wants \emph{and} what they are avoiding.
\end{itemize}

\textbf{Example Query:} \\
``I need a black dress that's more casual than my fitted Star knot velour dress but still looks put-together for spring. Something in that Tan dress style but in black since I'm tired of always wearing the same heavy pieces.''

\vspace{0.4em}
\textbf{Example Item Context:} \\
``Spring-transition lightweight chiffon; casual-elegance alternative to heavy velour Star knot dress; work-to-weekend versatility; wrapover styling; breathable fabric; black foundation expands options; polished-casual aesthetic reduces outfit repetition.''

\vspace{0.4em}
\textbf{Output Format:}
\begin{verbatim}
query1: <customer search query>
item_context1: <comprehensive matching description>
target_item_id1: <ITEM_ID>
[repeat for query2, query3, query4, query5]
\end{verbatim}
\end{tcolorbox}
\caption{Prompt used for H\&M contextualized query generation using the Claude Sonnet 3 expert model.}
\label{fig:hnm_prompt_example}
\end{figure}

 \subsection{Additional Experimental Details}
\label{app:baselines}

\paragraph{Task-specific Prompts}
\Cref{fig:hnm_fashion_prompt} and \Cref{fig:esci_search_prompt} shows the full prompts used for training and evaluation of the conventional product search task (ESCI) and contextualized product search task (H\&M Fashion Recommendations) reported in \Cref{tab:combined_search_results_esci_hnm}. \Cref{fig:hnm_summary_fullprompt} shows the full prompt used for user profile summary generation on H\&M Fashion data, reported in \Cref{tab:ospo_wr_compact}.

\paragraph{Baselines}
All trainable baselines reported in \Cref{tab:combined_search_results_esci_hnm}—\textbf{SFT}, \textbf{DPO}, \textbf{GRPO}, and \textbf{OSPO}—use the \textbf{Qwen2.5-Instruct 7B} model as the base architecture. For each domain, we also evaluate non-trainable encoder-based retrieval systems as task-specific baselines for contextualized search prior to our main experiments. Based on these preliminary experiments, we selected \textsc{all-mpnet-base-v2} for the ESCI shopping task and \textsc{SIMCSE-Large} for the fashion domain, as they achieved the highest retrieval NDCG on their respective validation splits when used for FAISS-based dense retrieval.

\textbf{Supervised Finetuning (SFT)} is trained using expert-generated samples filtered by task performance, retaining only those with NDCG $\geq 0.5$ for user summarization and $\geq 0.3$ for search tasks.
\textbf{Direct Preference Optimization (DPO)}~\citep{rafailov2024direct} learns from contrastive pairs sampled from the same expert generation process: “winning” responses are those with higher evaluation scores (NDCG or AP), and “losing” ones are those at least 5\% lower in AP or NDCG.
\textbf{Group-Relative Policy Optimization (GRPO)}~\citep{shao2024deepseekmath} is used as the representative value-free RL baseline, employing identical scalar reward functions as \textsc{OSPO}: retrieval NDCG for contextualized search and reward model scores for user summarization.
This unified setup ensures that all baselines differ only in how they learn from expert-generated feedback—contrastively (DPO), directly (SFT), or via sequence-level gradient updates (GRPO/OSPO).

\paragraph{Attention-based methods} Attention-based credit redistribution methods~\citep{chan2024dense} require access to reward model attention weights and tokenizer alignment between reward and policy models. Since the paper focuses on product search tasks that uses a dense retriever rather than a parametric reward model, such methods are not directly applicable. As such, we did not compare \ospo{} against attention-based baselines in our main experiments. Instead, we validate attribution quality directly in \Cref{tab:phrase_correlation} (\Cref{ssec:compute_ospo_appendix}): Owen values are $3.3\times$ more correlated with ground-truth solo-phrase retrieval quality ($\rho = 0.289$ vs.\ $0.078$), supporting coalition-based marginal attribution over attention-based proxies.

\paragraph{Format Rewards for User Summarization} While search tasks use only retrieval-based NDCG rewards, the user summarization task benefits from additional format validation. For this task, both \textsc{OSPO} and \textsc{GRPO} use a combined scalar reward with two components: a Bradley--Terry (BT) reward model (90\% weight) and a lightweight format-based reward (10\% weight). The BT reward\footnote{BT scores are normalized to $[0,1]$ using stable min--max normalization $(r - r_{\min}) / (r_{\max} - r_{\min})$, with empirical global bounds $r_{\min}=-5$ and $r_{\max}=10$ estimated over the dataset.} evaluates the quality of user summaries in the \texttt{<thinking>} section based on contrastive training with expert samples (See Appendix~\ref{app:expert_data_generation_prompts}). The format-based reward ensures syntactic correctness of ranked outputs (e.g., \texttt{<ranking>[800691007],[800691016],[752512008]</ranking>}).

This combined reward forms the sequence-level terminal reward $r^{(g)}$ in \Cref{alg:ospo}. When applying PBRS~\citep{ng1999policy} normalization at the token level (Step 16), the efficiency property ($\sum_{t=1}^{T} \phi_t = r^{(g)}$) ensures the sum of token-level attributions equals the total reward. Since PBRS simply scales attributions to match the sequence reward, using a weighted combination of rewards preserves the theoretical guarantees of Owen-value attribution while providing practical training benefits.

\paragraph{Candidate Pool Construction}
To stabilize RL training during product search tasks, we sample query-specific candidate pools that densify reward signals and are shared across GRPO and \ospo{} for fairness. During evaluation, the same filtered setting is used for consistency. For \textbf{ESCI}, we use the US Exact-Match subset, keeping queries with 3–10 relevant products. Direct NDCG over top-1000 FAISS results proved too sparse, so we build 100-item pools per query containing ground-truths, hard negatives (same category), related-category distractors, and random items. This yields denser, more informative rewards.
For \textbf{H\&M}, each query has one ground-truth target; we form analogous 100-item pools with that target plus 99 negatives sampled from similar categories.

\vspace{-3mm}
\paragraph{Hyperparameters} Our codebase is based on the TRL library. All RL training runs use \texttt{Qwen2.5-7B-Instruct} with learning rate $5\mathrm{e}{-6}$, batch size 8, and gradient accumulation over 2 steps for 2000 updates. Maximum prompt/completion lengths are 256/356 tokens for ESCI, 356/512 for H\&M, and 900/800 for user summarization. These settings are shared across \ospo{} and GRPO. Following~\citep{lin2025rec}, we set $\beta{=}0.0$ for KL regularization and use FAISS with top-$k{=}1000$ retrieval for reward computation. We additionally filter the FAISS-retrieved top-k items based on a candidate pools of relevant items, collected prior to training in each domain. Each prompt samples 8 generations for stable advantage estimates. For \ospo{}, coalition width $w_{\max}{=}8$ and 64 permutations per sequence balance local context and diversity; only \textit{contiguous} coalitions are used for semantically coherent Owen value assignments (\Cref{alg:ospo}).

\subsection{Additional Results}
\label{app:additional_results}

\paragraph{Comparison of \ospo{} with standard RL Baselines and GPT-4o on H\&M Product Search}
\Cref{tab:hnm_small_models_full} shows \ospoProp{} performance using Qwen 2.5 Instruct (1.5B and 3B models) against standard RL baselines including value-model based methods such as PPO~\citep{schulman2017proximalpolicyoptimizationalgorithms} as well as value-model free approaches like GRPO-$\lambda$ that explicitly target the uniform credit assignment problem in GRPO. Results are reported across 3 independent evaluation runs. Notably, all RL methods—initialized with the SFT model—improve over their respective 
zero-shot base models, with gains increasing monotonically from 
uniform credit assignment (GRPO) through temporal credit 
(GRPO-$\lambda$) and value-model credit (PPO) to segment-level 
marginal attribution (OSPO), confirming that credit assignment 
quality is the primary driver of post-training gains at 
deployment-relevant scales. At 1.5B, OSPO-Prop achieves 
$0.361 \pm 0.008$ NDCG, outperforming PPO ($0.333 \pm 0.005$) 
despite requiring no parametric critic, and exceeding GRPO 
($0.307 \pm 0.005$) by a statistically significant margin 
($p < 0.05$, paired $t$-test across 1{,}005 evaluation queries). At 3B, 
GRPO degrades below SFT ($0.385$ vs.\ $0.390$) under reward 
variance collapse—consistent with trends observed in prior work for similar scale models~\citep{samineni2025rlonlyanalyzingstructural}—while OSPO-Prop ($0.422 \pm 0.002$) leads 
all methods, reflecting stable and 
consistent gains across evaluation runs. As a zero-shot upper 
bound, we additionally report performance of a strong model like GPT-4o~\citep{openai2024gpt4ocard}. GPT-4o achieves $0.433 \pm 0.009$ NDCG on H\&M. Notably, 
OSPO-Prop at 3B ($0.422$) approaches this frontier, demonstrating that 
segment-level credit assignment enables small fine-tuned models 
to compete with frontier zero-shot performance on contextualized 
fashion search.

\begin{table}[t]
\centering
\small
\setlength{\tabcolsep}{4pt}
\renewcommand{\arraystretch}{0.95}

\begin{tabular}{lcccc}
\toprule
\textbf{Model} & \textbf{NDCG} & \textbf{AP} & \textbf{MRR} & \textbf{Recall} \\
\midrule
\multicolumn{5}{l}{\textit{Zero-shot Baselines}} \\
Qwen2.5-1.5B  & 0.304\tiny{$\pm$.003} & 0.287\tiny{$\pm$.002} & 0.287\tiny{$\pm$.002} & 0.353\tiny{$\pm$.003} \\
Qwen2.5-3B    & 0.384\tiny{$\pm$.001} & 0.367\tiny{$\pm$.001} & 0.367\tiny{$\pm$.001} & 0.434\tiny{$\pm$.002} \\
GPT-4o   & \textbf{0.433}\tiny{$\pm$.009} & \textbf{0.413}\tiny{$\pm$.009} & \textbf{0.413}\tiny{$\pm$.009} & \textbf{0.495}\tiny{$\pm$.011} \\
\midrule
\multicolumn{5}{l}{\textit{1.5B fine-tuned}} \\
SFT           & 0.306\tiny{$\pm$.003} & 0.289\tiny{$\pm$.003} & 0.289\tiny{$\pm$.003} & 0.356\tiny{$\pm$.004} \\
GRPO          & 0.307\tiny{$\pm$.005} & 0.289\tiny{$\pm$.005} & 0.289\tiny{$\pm$.005} & 0.359\tiny{$\pm$.007} \\
GRPO-$\lambda$    & 0.317\tiny{$\pm$.006} & 0.299\tiny{$\pm$.006} & 0.299\tiny{$\pm$.006} & 0.366\tiny{$\pm$.007} \\
PPO  & 0.333\tiny{$\pm$.005 }    & 0.316\tiny{$\pm$.006} & 0.316\tiny{$\pm$.006} & 0.386\tiny{$\pm$.005} \\
\textbf{OSPO-Prop}   & \textbf{0.361}\tiny{$\pm$.008} & \textbf{0.342}\tiny{$\pm$.009} & \textbf{0.342}\tiny{$\pm$.009} & \textbf{0.417}\tiny{$\pm$.008} \\
 
\midrule
\multicolumn{5}{l}{\textit{3B fine-tuned}} \\
SFT           & 0.390\tiny{$\pm$.003} & 0.373\tiny{$\pm$.003} & 0.373\tiny{$\pm$.003} & 0.441\tiny{$\pm$.004} \\
GRPO           & 0.385\tiny{$\pm$.008} & 0.368\tiny{$\pm$.008} & 0.368\tiny{$\pm$.008} & 0.434\tiny{$\pm$.009} \\
GRPO-$\lambda$    & 0.393\tiny{$\pm$.007} & 0.376\tiny{$\pm$.007} & 0.376\tiny{$\pm$.007} & 0.443\tiny{$\pm$.008} \\

PPO & 0.403\tiny{$\pm$.003}    & 0.385\tiny{$\pm$.003} & 0.385\tiny{$\pm$.003} & 0.456\tiny{$\pm$.002} \\
\textbf{OSPO-Prop}      & \textbf{0.422}\tiny{$\pm$.002} & \textbf{0.404}\tiny{$\pm$.002} & \textbf{0.404}\tiny{$\pm$.002} & \textbf{0.475}\tiny{$\pm$.002} \\
\bottomrule
\end{tabular}
\caption{H\&M Fashion search results across smaller (1.5 and 3B) model scales (mean $\pm$ SEM across 3 independent evaluation runs). Highest-performing zero-shot and finetuned baselines are shown in \textbf{bold}.}
\label{tab:hnm_small_models_full}
\end{table}

\paragraph{Generalization Under Retriever Shift.}
\noindent
\Cref{tab:cross_retriever_results} and \Cref{tab:combined_search_results_esci_hnm} together evaluate the \textit{cross-retriever robustness} of alignment methods—how well trained policies maintain ranking quality when tested with a different embedding and retrieval backbone. Specifically, \textsc{GRPO} and \ospo{} were trained on ESCI using \textsc{all-mpnet-base-v2} embeddings and FAISS indices, and are re-evaluated with \textsc{SIMCSE-Large}; the reverse is applied to the H\&M domain. \textit{This bi-directional setup tests how policies generalize when exposed to new embedding geometries and reward spaces unseen during training.}

\noindent
As shown in \Cref{tab:cross_retriever_results}, \textbf{\ospoProp{}} demonstrates strong generalization on ESCI product search, substantially outperforming all baselines across NDCG, MRR, and AP, confirming its ability to preserve task-specific robustness and quality with fine-grained credit assignment even under embedding shifts. In the H\&M product search task, \ospo{} additionally maintains consistent gains over \textsc{GRPO}, though the absolute values are lower due to cross-domain retriever mismatch. The “offline” baselines (\textsc{DPO}, \textsc{SFT}) perform moderately well in-domain but degrade sharply under retriever shift, underscoring their reliance on static supervision without retrieval-grounded feedback. Similarly, larger Qwen models (14B, 72B) show limited transfer, reinforcing that scale alone does not ensure robustness across retrieval domains. 

When comparing against the in-domain results from \Cref{tab:combined_search_results_esci_hnm}, we observe the expected performance drop under cross-retriever evaluation. For example, \textbf{\ospoProp{}} in H\&M—drops from $\text{NDCG}=0.436$ and $\text{Recall}=0.486$ to $0.190$ and $0.240$ respectively, while \textsc{GRPO} declines from $0.380$ to $0.184$. Despite this degradation, \ospo{} variants consistently retain higher relative performance and more stable recall than \textsc{GRPO}, illustrating its improved generalization and robustness to embedding-space and reward-model variations. In the ESCI domain product search, the same pattern holds: all methods experience performance loss when tested with the alternate retriever, but \textbf{\ospoProp{} remains the strongest overall}, surpassing both reinforcement and supervised baselines even under out-of-distribution (OOD) retriever settings.

\begin{table*}[t]
\centering
\small
\setlength{\tabcolsep}{4.5pt}
\begin{tabular}{lcccc|cccc}
\toprule
\multirow{2}{*}{\textbf{Baseline}} & \multicolumn{4}{c|}{\textbf{ESCI}} & \multicolumn{4}{c}{\textbf{H\&M}} \\
\cmidrule(lr){2-5} \cmidrule(lr){6-9}
 & \textbf{NDCG} & \textbf{AP} & \textbf{MRR} & \textbf{Recall} & \textbf{NDCG} & \textbf{AP} & \textbf{MRR} & \textbf{Recall} \\
\midrule
SFT (7B)                & 0.1761 & 0.1158 & 0.4004 & 0.1166 & -- & -- & -- & -- \\
DPO (7B)       & 0.2280 & 0.1325 & 0.4412 & 0.1330 & -- & -- & -- & -- \\
GRPO         & 0.2257 & 0.1144 & 0.3887 & 0.1150 & 0.1844 & 0.1681 & 0.1681 & 0.2348 \\
\midrule
Qwen2.5-0.5B            & 0.2183 & 0.1121 & 0.3757 & 0.1127 & 0.1453 & 0.1323 & 0.1323 & 0.1851 \\
Qwen2.5-1.5B            & 0.2926 & 0.1382 & 0.4648 & 0.1390 & 0.2158 & 0.1986 & 0.1986 & 0.2687 \\
Qwen2.5-3B              & 0.2109 & 0.1209 & 0.4128 & 0.1217 & 0.1504 & 0.1392 & 0.1392 & 0.1851 \\
Qwen2.5-14B             & 0.2628 & 0.1583 & 0.5202 & 0.1592 & -- & -- & -- & -- \\
Qwen2.5-32B             & 0.2549 & 0.1581 & 0.5135 & 0.1597 & 0.2025 & 0.1836 & 0.1836 & 0.2607 \\
Qwen2.5-72B             & \textbf{0.2986} & \textbf{0.1748} & \textbf{0.5766} & \textbf{0.1759} & 0.1621 & 0.1446 & 0.1446 & 0.2159 \\
\midrule
\textbf{OSPO (Prop)} \textit{(ours)} & \underline{0.3691} & \underline{0.1673} & \underline{0.5395} & \underline{0.1683} & \textbf{0.1901} & \textbf{0.1739} & \textbf{0.1739} & \textbf{0.2398} \\
\textbf{OSPO (Rank)} \textit{(ours)} & 0.2675 & 0.1468 & 0.4888 & 0.1475 & 0.2019 & 0.1846 & 0.1846 & 0.2557 \\
\bottomrule
\end{tabular}
\caption{\textbf{Cross-Retriever Evaluation on Product Search Tasks.}
RL baselines GRPO and \ospo{} were trained on the ESCI domain using dense retrieval rewards computed with \textsc{all-mpnet-base-v2} embeddings and a FAISS HNSW index for similarity search. For evaluation, we test generalization under a retriever shift by re-encoding sampled refined queries in both ESCI and H\&M evaluation input queries using \textsc{SIMCSE-Large}. This setting measures \textit{cross-retriever robustness}—how well alignment methods such as \textsc{OSPO} and GRPO preserve ranking quality when the embedding space and retrieval backbone differ from those used during training.}
\label{tab:cross_retriever_results}
\end{table*}

\paragraph{Generalization Beyond Retrieval: IMDB Sentiment Control.}
To evaluate whether \ospo{}'s gains transfer outside personalized 
recommendation settings, we test it on a standard RLHF sentiment 
generation benchmark. We fine-tune \texttt{lvwerra/gpt2-imdb} as the base model
using \texttt{lvwerra/distilbert-imdb} as a reward model in TRL\footnote{\url{https://huggingface.co/docs/trl/v0.11.4/en/sentiment_tuning}}, 
scoring completions by the logit of the positive sentiment class. 
All methods share identical training configuration: learning rate 
$5{\times}10^{-6}$, 3 independent seeds, and the same base model. 
We compare \ospoProp{} against GRPO~\citep{shao2024deepseekmath}, 
GRPO-$\lambda$~\citep{parthasarathi2025grpolambdacreditassignmentimproves}, 
and DAPO~\citep{yu2025dapo}—all value-model-free baselines for a direct comparison in critic-free settings. 

Fig.~\ref{fig:imdb_training_dynamics} shows training dynamics across 1{,}500 
steps. \ospoProp{} converges faster and reaches substantially higher 
train and eval rewards than all baselines, while exhibiting 
significantly lower reward variance (bottom-left) after step 500. This
indicates more stable and consistent policy learning, while GRPO collapses 
to near-zero completion length (bottom-right) by step 300, a 
degenerate reward-hacking solution where the model learns to generate 
empty or minimal sequences to avoid negative sentiment tokens—a 
direct manifestation of the uniform credit assignment failure 
documented by~\citet{samineni2025rlonlyanalyzingstructural}. 
In contrast, GRPO-$\lambda$ and DAPO partially mitigate this collapse through 
eligibility traces and dynamic sampling respectively, but neither 
achieves \ospoProp{}'s reward level or stability. These results 
confirm that Owen-Shapley marginal attribution provides a more 
effective credit signal than temporal or sequence-level alternatives 
in controlled generation settings. Notably, this faster learning trends in \ospo{} showing its sample-efficiency is consistent with our main results, where \ospo{} achieves substantially higher performance at lower training steps vs GRPO (See \Cref{fig:ospo_overview}; bottom right).

\begin{figure*}[t]
\centering
\setlength{\tabcolsep}{2pt}

\begin{tabular}{cc}
\includegraphics[width=0.48\linewidth]{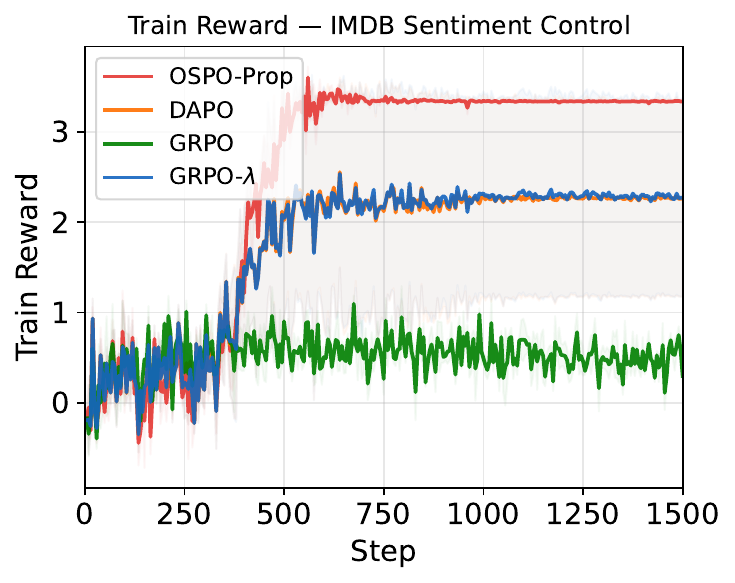} &
\includegraphics[width=0.48\linewidth]{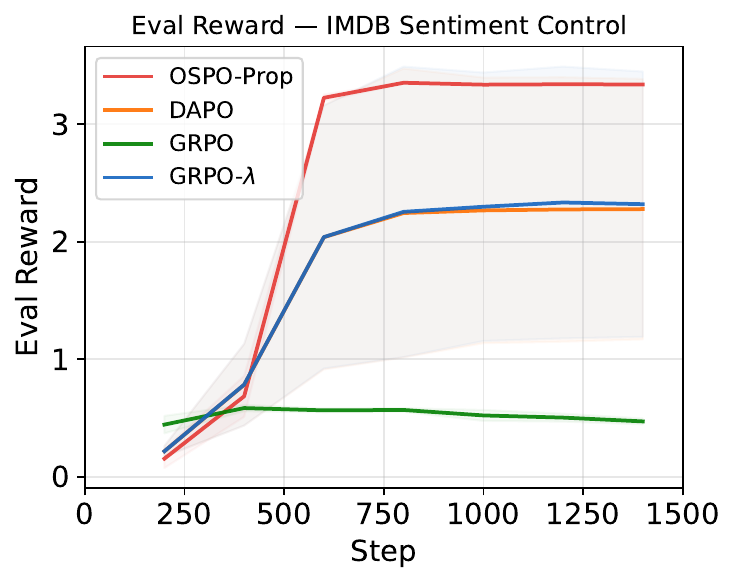} \\

\includegraphics[width=0.48\linewidth]{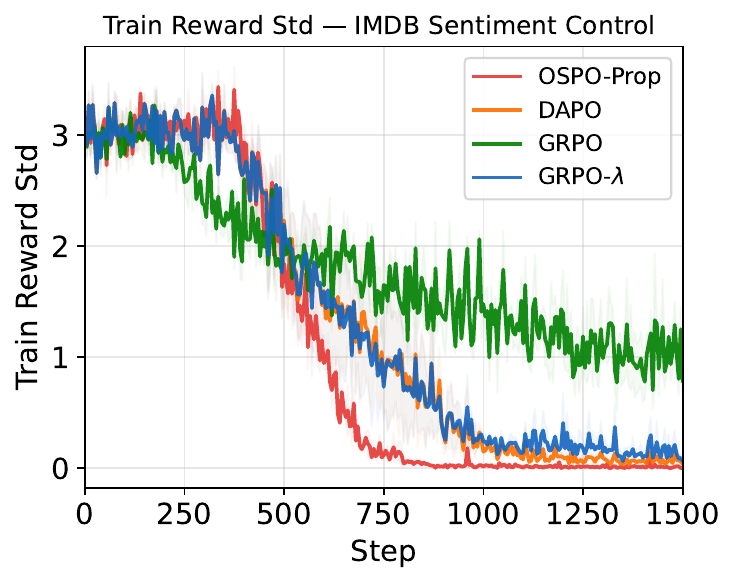} &
\includegraphics[width=0.48\linewidth]{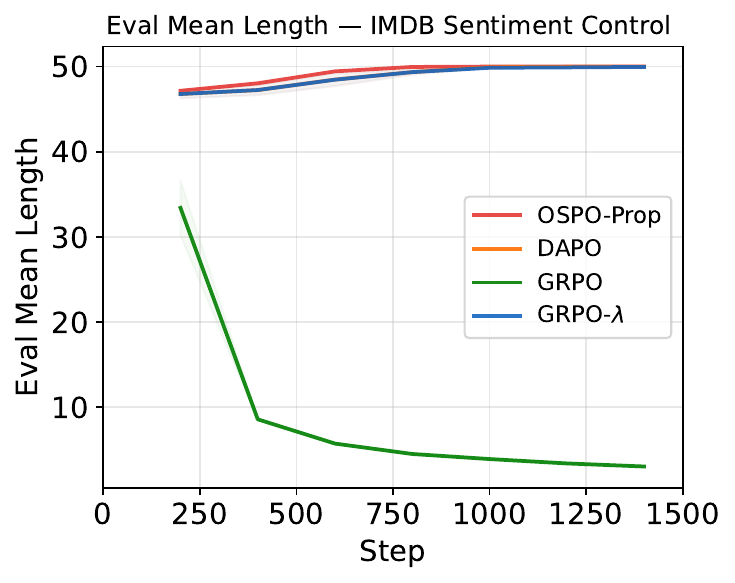} \\
\end{tabular}

\caption{
\textbf{Training dynamics of \ospo{} vs value-model free RL baselines on IMDB sentiment control}(lr=$5\times10^{-6}$) across 3 seeds. 
\textbf{Top-left:} Train reward. \textbf{Top-right:} Evaluation reward. 
\textbf{Bottom-left:} Reward standard deviation. \textbf{Bottom-right:} Mean completion length. 
\ospoProp{} converges faster and reaches substantially higher rewards than GRPO, GRPO-$\lambda$~\citep{parthasarathi2025grpolambdacreditassignmentimproves}, and DAPO~\citep{yu2025dapo}, while also exhibiting significantly lower variance, indicating more stable and consistent policy learning. These results highlight the effectiveness of marginal (Owen-Shapley) credit assignment over temporal or sequence-level alternatives in controlled generation settings.
}
\label{fig:imdb_training_dynamics}
\end{figure*}

\paragraph{Compute Budget and GPU Setup}
All \textbf{7B scale experiments} were conducted on 8$\times$H200 GPUs. Each OSPO baseline
required approximately 24 hours to complete 2{,}000 
training steps on both ESCI (9{,}339 training samples) and H\&M datasets (24{,}960 training samples). OSPO required 
approximately $1.6\times$ the wall-clock time of GRPO for the same 
number of steps, consistent with the 5.8\% per-step overhead reported 
in \Cref{ssec:compute_ospo_appendix} accumulated over the full Owen 
coalition evaluation pipeline. These results are reported in \Cref{tab:combined_search_results_esci_hnm}. Key 7B hyperparameters: LoRA ($r{=}32$, 
$\alpha{=}16$, dropout$=0.05$), 4-bit quantization, bfloat16, 
paged AdamW ($\text{lr}{=}5{\times}10^{-6}$, cosine schedule, 
100 warmup steps), effective batch size 16 
(per-device$=8$, gradient accumulation$=2$), $G{=}8$ rollouts, 
max prompt/completion lengths of 256/356 tokens, 
$w_{\max}{=}8$, $p{=}64$, KL coefficient $\beta{=}0$. 
 
For \textbf{1.5B and 3B scale experiments} on H\&M Product Search (\Cref{tab:hnm_wrap_credit_assignment}), all methods are trained on an NVIDIA Blackwell GPU (96GB VRAM)  
for 500 steps (${\approx}$1 full epoch over 8k train prompts) with effective batch size 16, using 8{,}000 H\&M training prompts out of total 24{,}960 training samples, initialized from the same SFT checkpoint (LoRA $r{=}32$, 
1 epoch, assistant-only loss, 1{,}966 training pairs), KL coefficient $\beta{=}0.1$\footnote{Unlike 7B scale training runs, we initialize RL training with KL coefficient 
$\beta{=}0.1$ instead of zero, since training is conducted with the SFT'ed checkpoint as the starting point.}, 
and NDCG@1000 with candidate pool filtering as the reward signal. GRPO and \ospo{} share learning rate $5{\times}10^{-6}$, cosine schedule with 100 
warmup steps, per-device batch size 8, gradient accumulation 2, $G{=}8$ rollouts 
per prompt, LoRA ($r{=}32$, $\alpha{=}16$, dropout$=0.05$), 4-bit quantization, 
bfloat16, and paged AdamW optimizer. \ospo{} additionally uses contiguous coalition 
evaluation with $w_{\max}{=}8$ and budget cap $p{=}96$. GRPO-$\lambda$~\citep{parthasarathi2025grpolambdacreditassignmentimproves} shares all GRPO hyperparameters 
with the addition of decay parameter $\lambda{=}0.99$ applied to group-relative advantages using the optimal $\lambda$ value as reported in their paper. For a fair evaluation, PPO~\citep{schulman2017proximalpolicyoptimizationalgorithms} shares the same LoRA 
configuration, batch setup, quantization, and optimizer as GRPO and \ospo{}, but 
with learning rate $1{\times}10^{-6}$ consistent with prior work~\citep{kazemnejad2024vineppo}, 4 PPO epochs per rollout batch, GAE with 
$\lambda{=}0.95$, KL coefficient $0.1$ with target KL $6.0$, clip range 
$\varepsilon{=}0.2$, and value function coefficient $0.1$; the value head is 
initialized from the same SFT checkpoint.

\subsection{Coalition Formation and Sampling Strategy}
\label{ssec:coalition_formation}

Our approach to computing Owen-Shapley~\citep{owen1977values, owen2014sobol} attributions requires evaluating the marginal contribution of each segment across multiple coalitions. We now detail how we construct these coalitions and manage the computational trade-offs inherent in this process.

\paragraph{Contiguous Coalition Generation}
Given a response $y^{(g)}$ sampled from the policy $\pi_\theta$ during \ospo{} training, we decompose it into $N$ segments $\mathcal{A}^{(g)} = \{a_1, \ldots, a_N\}$ via linguistic phrase extraction for product search tasks (where segments are clauses) or regex-based sentence splitting for user summary generation tasks. We then generate contiguous coalitions—subsets of consecutive segments that preserve the sequential structure of the response. Formally, a coalition $S$ is \emph{contiguous} if $S = \{i, i+1, \ldots, i+w-1\}$ for some starting position $i \in \{1, \ldots, N\}$ and width $w \in \{1, \ldots, \min(w_{\max}, N)\}$. This contiguity constraint ensures that coalition-based partial segments $y_S$ remain semantically coherent—a critical requirement for meaningful reward evaluation in our tasks. 

The complete set of contiguous coalitions $\mathcal{C}$ includes:
\begin{itemize}[leftmargin=*,itemsep=2pt]
    \item The \textbf{empty coalition} $\emptyset$, representing no query (baseline reward\footnote{In practice, for the user summary generation task—where the reward function $r(x, y_{S_m})$ depends on the context $x$—we allow for negative baseline rewards. Note that rewards for search tasks that are assigned by a retriever acting as a reward model only depend on the partial query $y_{S_m}$ and not on the context.} of 0)
    \item All \textbf{contiguous windows} of width $w = 1, \ldots, \min(w_{\max}, N)$
    \item The \textbf{full coalition} $\{1, \ldots, N\}$, representing the complete response
\end{itemize}

The total number of possible contiguous coalitions is:
\begin{equation}
\label{eq:coalition_count}
|\mathcal{C}| = 1 + \sum_{w=1}^{\min(w_{\max}, N)} (N - w + 1)
\end{equation}
where the summation counts all starting positions for each width $w$, and the leading $1$ accounts for the empty coalition. For example, with $N=5$ segments and $w_{\max}=3$:
\begin{itemize}[leftmargin=*,itemsep=2pt]
    \item Width 1: $\{1\}, \{2\}, \{3\}, \{4\}, \{5\}$ (5 coalitions)
    \item Width 2: $\{1,2\}, \{2,3\}, \{3,4\}, \{4,5\}$ (4 coalitions)
    \item Width 3: $\{1,2,3\}, \{2,3,4\}, \{3,4,5\}$ (3 coalitions)
\end{itemize}
yielding $|\mathcal{C}| = 1 + 5 + 4 + 3 = 13$ total coalitions including the empty and full coalitions.

\paragraph{Computational Budget and Sampling.}
Computing rewards for all coalitions becomes expensive and computationally intractable as $N$ grows, since $|\mathcal{C}| = O(N \cdot \min(w_{\max}, N))$. To maintain tractability, we impose a sampling budget $M$ (see \Cref{alg:ospo}) that caps the number of coalitions evaluated per response. The actual number of sampled coalitions is:
\begin{equation}
\label{eq:sampled_coalitions}
k = \min(M, |\mathcal{C}|)
\end{equation}

When $|\mathcal{C}| \leq M$, we evaluate all possible coalitions, providing exact Owen-Shapley values. When $|\mathcal{C}| > M$, we sample $M$ coalitions using the following strategy:
\begin{enumerate}[leftmargin=*,itemsep=2pt]
    \item \textbf{Anchor coalitions}: Always retain $\emptyset$ and $\{1, \ldots, N\}$, as these provide essential boundary conditions for marginal contribution estimates
    \item \textbf{Random sampling}: Uniformly sample $M-2$ coalitions from the remaining $|\mathcal{C}| - 2$ candidates
\end{enumerate}

This sampling strategy ensures that our Owen value estimates remain well-calibrated: the empty coalition anchors the baseline, while the full coalition captures the terminal reward. The intermediate coalitions provide estimates of each segment's contribution across diverse contexts.

\paragraph{Coalition Evaluation}
For each sampled coalition $S_m \in \{S_1, \ldots, S_k\}$, we construct the partial query $y_{S_m}$ by concatenating the corresponding segments in order:
\begin{equation}
y_{S_m} = a_i \oplus a_{i+1} \oplus \cdots \oplus a_{i+w-1} \quad \text{where } S_m = \{i, i+1, \ldots, i+w-1\}
\end{equation}
and evaluate its coalition value $v(S_m) = r(x, y_{S_m})$ using our reward function. In our search task, this involves issuing $y_{S_m}$ as a query to the retrieval system and computing the NDCG@1000 score against the target item(s). The empty coalition receives $v(\emptyset) = 0$ by convention, reflecting the absence of information.

\paragraph{Contiguity Rationale}
Our choice to restrict coalitions to contiguous subsets—rather than exploring all $2^N$ possible subsets—is motivated by both computational efficiency and semantic coherence. Non-contiguous coalitions (e.g., $\{1, 3, 5\}$ skipping segments 2 and 4) would produce fragmented queries that (a) require $O(2^N)$ evaluations, and (b) may yield semantically incoherent text, confounding reward evaluation and possibly lead to performance degradation due to the loss of local structure and high-variance Owen value estimates. Contiguity preserves the narrative flow of the response while reducing complexity to $O(N \cdot w_{\max})$, making Owen value estimation tractable for online RL training. In fact, our ablation study on coalition sampling in \Cref{ssec:coalition_ablations} suggests that non-contiguity in sampling can lead to a significant performance drop in conventional product search task on ESCI dataset.  

\paragraph{Illustrative Example}
Consider a response $y^{(g)}$ with extracted phrases:
\begin{align*}
a_1 &= \text{``blue midi dress''} \\
a_2 &= \text{``summer wedding''} \\
a_3 &= \text{``elegant style''}
\end{align*}
With $w_{\max}=4$ and $M=32$, we generate $|\mathcal{C}| = 1 + 3 + 2 + 1 = 7$ coalitions:
\begin{center}
\begin{tabular}{ll}
\toprule
Coalition $S_m$ & Partial Query $y_{S_m}$ \\
\midrule
$\emptyset$ & (empty, $v(\emptyset) = 0$) \\
$\{1\}$ & ``blue midi dress'' \\
$\{2\}$ & ``summer wedding'' \\
$\{3\}$ & ``elegant style'' \\
$\{1,2\}$ & ``blue midi dress summer wedding'' \\
$\{2,3\}$ & ``summer wedding elegant style'' \\
$\{1,2,3\}$ & ``blue midi dress summer wedding elegant style'' \\
\bottomrule
\end{tabular}
\end{center}
Since $|\mathcal{C}| = 7 < M = 32$, all coalitions are evaluated. Each query is submitted to the search system, producing coalition values $\{v(S_m)\}_{m=1}^7$. For instance, if the target item is a blue dress suitable for weddings, we might observe $v(\{1\}) = 0.6$, $v(\{1,2\}) = 0.85$, and $v(\{1,2,3\}) = 0.82$. The marginal contribution of segment $a_2$ in the context of $S = \{1\}$ is then $v(\{1,2\}) - v(\{1\}) = 0.25$, quantifying how ``summer wedding'' refines the query when added to ``blue midi dress.'' These marginal contributions across all applicable coalitions are aggregated via the Owen-Shapley formula to produce the final attribution $\phi_2$ for segment $a_2$.

\begin{figure}[t]
\centering
\includegraphics[width=\linewidth]{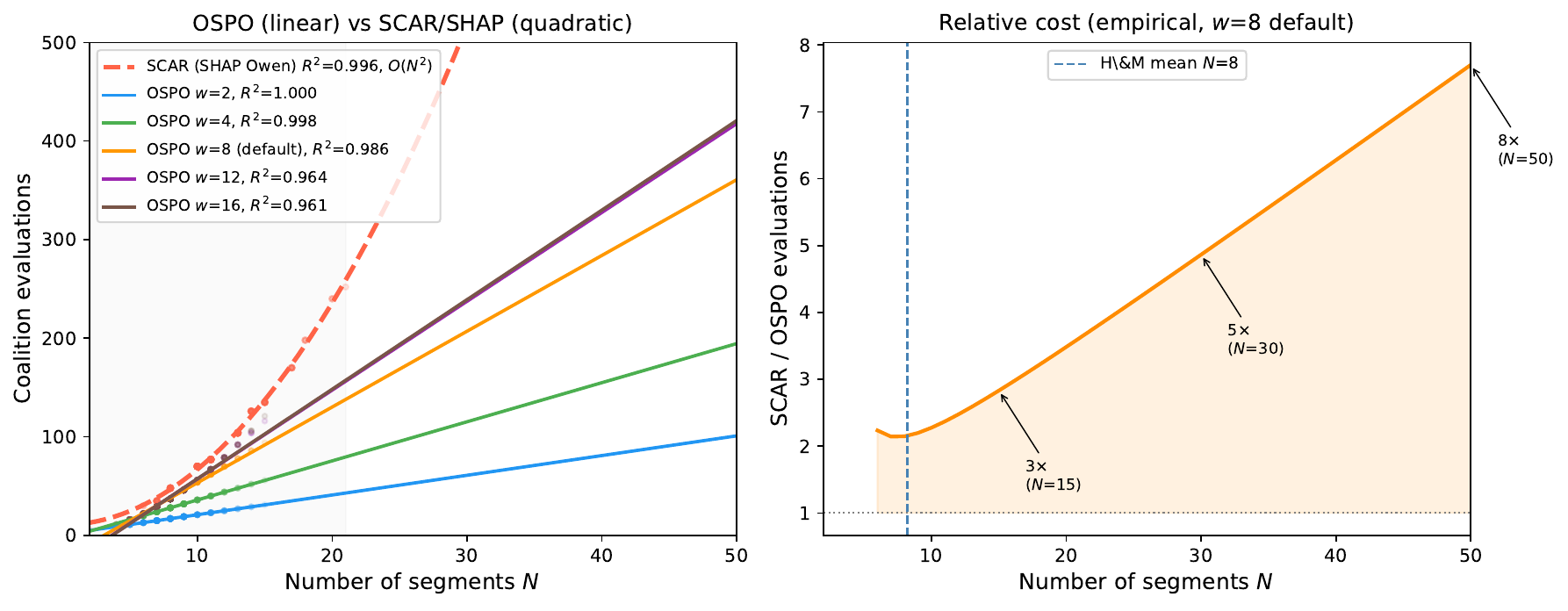}
\caption{
\textbf{Coalition complexity: \ospo{} (linear) vs.\ SCAR/SHAP}  (quadratic).
\textit{Left:} Coalition evaluation counts as a function of segment 
count $N$, measured on \textbf{Qwen2.5-7B-Instruct} rollouts on H\&M product search task. SCAR's SHAP-based partition~\citep{cao2025scar} Owen values fit a quadratic 
model ($R^2{=}0.996$);  \ospo{}'s contiguous-window restriction fits 
a linear model for all tested $w_{\max}$ ($R^2{\geq}0.961$). 
Higher $w_{\max}$ values ($\geq$8) converge on H\&M since 
$w_{\max} \geq \bar{N}$. Shaded region marks H\&M observed range.
\textit{Right:} Ratio of SCAR to  \ospo{} ($w{=}8$) coalition 
evaluations, extrapolated to $N{=}50$. The gap grows monotonically 
from $2.3\times$ at $\bar{N}{=}8$ to $8\times$ at $N{=}50$. We use the SHAP package (\url{https://shap.readthedocs.io/en/latest/generated/shap.PartitionExplainer.html}) to compute coalition generation in SCAR as noted in their work.
}
\label{fig:complexity}
\end{figure}

\paragraph{Comparison of \ospo{} vs. SCAR }
SCAR~\citep{cao2025scar} and \ospo{} share a game-theoretic foundation but target fundamentally different settings. SCAR requires a reward model capable of scoring partial completions—a condition standard reward models~\citep{lambert2024rewardbench} trained on full preference pairs do not satisfy. In contrast, \ospo{} targets retrieval-based rewards, where dense similarity models are trained on short underspecified queries~\citep{reddy2022shopping} (in our setup, ESCI query mean $\approx 3.8$ words) and naturally admit partial phrase coalitions as valid inputs. This task difference makes direct comparison inappropriate, since these two methods are addressing fundamentally different reward structures. Additionally, SCAR does not release a public RL training implementation. We therefore do not include SCAR in our baseline setup in our main experiments (\Cref{sec:experimental_setup}). Instead, we compare against PPO~\citep{schulman2017proximalpolicyoptimizationalgorithms} (\Cref{tab:hnm_wrap_credit_assignment} and \Cref{tab:hnm_small_models_full}) as the canonical critic-based baseline that uses GAE for credit assignment, since \ospo{}'s core goal is critic-free credit assignment that matches value-model quality.

We nevertheless provide a computational complexity comparison below, showing that \ospo{}'s contiguity-based coalition evaluation scales linearly whereas SCAR's partition-based Owen approximation scales quadratically.
\label{ssec:complexity_analysis} Fig.~\ref{fig:complexity} shows the results. Importantly, SCAR~\citep{cao2025scar} reduces exact Shapley's $O(2^N)$ to 
$O(N^2)$ via SHAP's partition-based Owen values~\citep{lundberg2017unified}, 
which decompose into an $O(N)$ tree traversal \textit{plus} a dominant 
$O(N^2)$ within-group subset enumeration. We empirically verify 
both scaling behaviors on coalition distributions derived from 
Qwen2.5-7B-Instruct rollouts on H\&M product search task
($\bar{N}{=}8.2$ phrases per response, $p_{95}{=}12$). Separately, we measure wall-clock cost of coalition evaluation in \ospo{} using the same rollouts, finding each FAISS lookup takes ~4ms, keeping total Owen overhead at 5.8\% of rollout time (\Cref{tab:wallclock}).

\textbf{SCAR scaling} We simulate SCAR's SHAP-based Owen 
computation by constructing a balanced binary partition tree 
over each response's phrase segments (using \texttt{scipy} 
Ward linkage, consistent with SHAP's \texttt{PartitionExplainer}), 
then counting the within-group subset evaluations. Fitting a 
quadratic model to the resulting coalition counts yields 
$R^2{=}0.996$, confirming the $O(N^2)$ claim from \citep{cao2025scar}
on our phrase distributions (mean 90.4 evaluations per response, 
vs.\ their stated $N^2{=}141.7$ upper bound at $\bar{N}{=}11.5$ 
constituency-parse players).

\textbf{\ospo{} scaling} \ospo{} restricts coalitions to contiguous 
spans of width $\leq w_{\max}$, yielding exactly 
$\sum_{w=1}^{w_{\max}}(N{-}w{+}1) = O(N w_{\max})$ evaluations. 
We fit a linear model across \textbf{five} configurations 
($w_{\max} \in \{2, 4, 8, 12, 16\}$); \textbf{all achieve 
$R^2 \geq 0.961$, confirming linear scaling empirically }
(Fig.~\ref{fig:complexity}). Notably, $w_{\max} \in \{8, 12, 16\}$ 
converge to nearly identical counts on H\&M (mean $\approx 40$), 
since the contiguous coalition space saturates when 
$w_{\max} \geq \bar{N}$. This confirms $w_{\max}{=}8$ 
is sufficient for our domain without loss of coverage as our ablations show (\Cref{ssec:coalition_ablations}).

\textbf{Relative cost} At H\&M scale ($N{\approx}8$), \ospo{} 
($w{=}8$) requires 39.6 coalition evaluations vs.\ SCAR's 
90.4 ($2.3\times$ reduction). Critically, each evaluation is 
a FAISS approximate nearest-neighbor lookup (${\approx}4$ms), 
not an LLM forward pass, keeping total Owen overhead at 
$5.8\%$ of rollout time. The computational gap widens 
substantially for longer-generation settings: $3\times$ at 
$N{=}15$, $5\times$ at $N{=}30$, and $8\times$ at $N{=}50$ 
(relevant for long-CoT or math-reasoning/summarization tasks), 
as illustrated in Fig.~\ref{fig:complexity} (right).

\subsection{Attribution Validity and Computational Analysis}
\label{ssec:compute_ospo_appendix}

A natural question for any credit assignment method is whether the 
attribution signal is both \textit{meaningful} and \textit{practical}. We address these in turn: first validating that Owen-Shapley credit 
recovers genuine retrieval causality rather than proxying simpler 
heuristics, then analyzing the computational overhead of coalition 
evaluation relative to standard GRPO~\citep{shao2024deepseekmath} training.

\paragraph{Attribution Validity: Owen-Shapley vs.\ Simpler Credit Proxies}
To verify that OSPO's gains stem from principled marginal contribution 
estimation rather than simpler attribution heuristics, we conduct a 
phrase-level correlation study using the same SimCSE retriever used 
during training (\Cref{sec:experimental_setup}). We randomly sample 200 queries from the H\&M evaluation set and generate 8 rollouts per prompt  using the trained 
\ospoProp{} model (Qwen2.5-7B-Instruct backbone) (1{,}600 generated queries total). For each 
query, linguistic phrases are extracted using spaCy with $w8-p96$ configuration and we compute: 
(1) the credit assigned to each phrase by each attribution method, 
and (2) the solo retrieval quality of each phrase---the NDCG obtained 
when that phrase alone is submitted to the FAISS index---which 
directly measures each phrase's independent retrieval contribution. 
Over 1{,}312 phrase-level observations, we compute the Spearman 
correlation between each method's assigned credit and ground-truth 
solo-phrase NDCG.

We compare five methods: \textbf{GRPO uniform} assigns identical 
weight to all tokens; \textbf{Random} provides an empirical noise 
floor; \textbf{TF-IDF} weights phrases by lexical specificity; 
\textbf{SimCSE attention} uses last-token or CLS-token attention 
weights from the retriever encoder as phrase credit proxies, 
following the approach of ABC~\citep{chan2024dense}; and 
\textbf{OSPO} computes each phrase's marginal NDCG contribution 
via contiguous coalition evaluation against the same retriever. 
Results are reported in Table~\ref{tab:phrase_correlation}.

\begin{table}[h]
\centering
\small
 
\begin{tabular}{lrrl}
\toprule
\textbf{Credit method} & $\bm{\rho}$ & \textbf{$p$-value} & 
\textbf{vs.\ Random} \\
\midrule
GRPO uniform                 & NaN            & ---       & --- \\
Random                       & 0.001          & 0.95      & baseline \\
TF-IDF                       & 0.079          & 0.004     & weak signal \\
SimCSE attn (ABC-CLS)        & 0.078          & 0.005     & weak signal \\
SimCSE attn (ABC-last token) & 0.078          & 0.004     & weak signal \\
\textbf{OSPO (Owen-Shapley)} & \textbf{0.289} & $<$0.0001 & 
\textbf{3.3$\times$ stronger} \\
\bottomrule
\end{tabular}
\caption{Spearman correlation ($\rho$) between phrase-level credit 
weights and solo-phrase retrieval quality (NDCG) across 1{,}312 
observations on H\&M. Higher $\rho$ indicates better alignment 
between assigned credit and true retrieval contribution.}
\label{tab:phrase_correlation}
\end{table}

Owen-Shapley credit is $3.3\times$ more correlated with ground-truth 
phrase retrieval quality than attention-based proxies 
($\rho{=}0.289$ vs.\ $0.078$), confirming that OSPO's gradient 
concentration reflects genuine retrieval causality rather than 
lexical frequency or encoder attention patterns. GRPO uniform 
receives NaN correlation by construction since identical weights 
provide no differential signal across phrases. TF-IDF and SimCSE 
attention weights show statistically significant but practically 
weak correlations ($\rho{\approx}0.078$), consistent with prior 
work showing attention weights are “heuristics-based” proxies for feature 
importance~\citep{cao2025scar}. These results directly motivate 
OSPO's use of coalition-based marginal attribution over 
attention-based redistribution in our main experiments 
(Section~\ref{sec:experimental_setup}).


\paragraph{Coalition Count and Budget Design}
The coalition budget parameter $p$ (e.g., $p{=}96$ in our default w8-p96 
configuration) serves as an upper bound on retrieval calls per response, not 
a fixed evaluation count. The \textit{actual} number of coalitions evaluated is 
$k = \min(p, |\mathcal{C}|)$, where $|\mathcal{C}|$ depends on the number 
of extracted phrases $N$ and the maximum coalition width $w_{\max}$. For example, in 
typical H\&M responses (N=5--8 phrases, $w_{\max}{=}8$), $|\mathcal{C}|$ 
ranges from 13--40, well below the cap of 96. Table~\ref{tab:coalition_budget} 
reports the mean actual evaluations and sampling trigger rate across all 
configurations evaluated in our ablation study (\Cref{ssec:coalition_ablations}) on coalition sampling.
Sampling is triggered only when the budget cap is set below the natural 
coalition count---the two configurations where this occurs (w3-p24, w6-p32) 
also exhibit lower task performance in Table~\ref{tab:ospo_ablation_esci}, confirming 
that noisy Owen estimates from undersampled coalitions degrade credit 
assignment quality. For all well-matched configurations, the budget cap is 
never or rarely reached, and coalitions are evaluated exhaustively.

\begin{table}[h]
\centering
\small
 
\begin{tabular}{lrrrr}
\toprule
\textbf{Config} & \textbf{Mean $|\mathcal{C}|$} & 
\textbf{Actual evals} & \textbf{\% sampling} \\
\midrule
w1-p16   & 10.2 & 10.2 & 0.3\% \\
w2-p64   & 17.3 & 17.3 & 0.0\% \\
w3-p24$^\dagger$ & 23.5 & 21.3 & 39.3\% \\
w4-p48   & 28.6 & 28.6 & 1.0\% \\
w6-p32$^\dagger$ & 35.8 & 28.9 & 60.7\% \\
w8-p96   & 39.2 & 39.2 & 0.0\%  \\
w12-p128 & 40.4 & 40.4 & 0.0\%   \\
w16-p256 & 40.5 & 40.5 & 0.0\% \\
\bottomrule
\end{tabular}

\caption{Coalition evaluation statistics across budget configurations on H\&M 
search responses. $\dagger$ denotes budget-mismatched configurations where 
$p < \overline{|\mathcal{C}|}$, triggering random sampling and producing 
noisier Owen estimates.}
\label{tab:coalition_budget}
\end{table}

\paragraph{Wall-Clock Breakdown}
Each coalition is evaluated via a single FAISS approximate nearest-neighbor 
lookup---one encode call followed by one index search---rather than a 
parametric reward model forward pass. Table~\ref{tab:wallclock} reports 
per-coalition and per-response timing measured across 3 independent runs 
of 200 H\&M samples each, using the w8-p96 configuration.

\begin{table}[h]
\centering
\small

\begin{tabular}{lr}
\toprule
\textbf{Component} & \textbf{Time} \\
\midrule
FAISS encode / coalition     & $3.99 \pm 0.009$ ms \\
FAISS search / coalition     & $0.13 \pm 0.004$ ms \\
Total FAISS / coalition      & $4.12 \pm 0.010$ ms \\
Owen overhead / response     & $160 \pm 1.8$ ms \\
Rollout time (G=8, Qwen-7B)  & $2749 \pm 19$ ms \\
\midrule
\textbf{Owen / Rollout overhead} & \textbf{5.8\%} \\
\bottomrule
\end{tabular}

\caption{Wall-clock timing breakdown for \ospo{} vs.\ GRPO on H\&M search 
(Qwen2.5-7B, G=8 rollouts, w8-p96). Owen computation occurs only during 
training; inference cost is identical to GRPO.}
\label{tab:wallclock}

\end{table}

The Owen attribution step adds $160$ms per response on average, representing 
a 5.8\% overhead relative to the rollout generation time of $2749$ms. 
\textit{Crucially, Owen-Shapley computation occurs only during training—inference 
cost is identical to GRPO since no coalition evaluation is performed at test 
time.} The practical overhead of \ospo{} relative to GRPO is therefore limited 
to training, where it is offset by faster convergence: as shown in 
Figure~\ref{fig:ospo_overview} (bottom-right), \ospo{} reaches target NDCG in approximately 
half the training steps required by GRPO.

\paragraph{Applicability to Other Reward Functions}
The low per-coalition cost reported above reflects \ospo{}'s use of dense 
retrieval (FAISS) as the characteristic function $v(S)$. Tasks using 
expensive LLM-as-judge reward functions would incur proportionally higher 
Owen computation costs. This is a deliberate design choice: OSPO targets 
outcome-based retrieval settings where coalition scoring is fast and 
semantically meaningful for partial queries.

\subsection{Ablations Study on Coalition Formation}
\label{ssec:coalition_ablations}
In order to identify the most effective coalition configuration for \ospo{} optimization, we conduct a series of experiments across multiple coalition settings to guide hyperparameter selection. To analyze the role of coalition structure more precisely, we perform additional ablations on both product search tasks by systematically varying two key parameters that govern Owen value estimation in \ospo{} as mentioned in \Cref{ssec:coalition_formation}. Specifically, \textbf{w} denotes the maximum coalition span ($w_{\text{max}}$), which determines the widest contiguous segments that can be included in a coalition, while \textbf{p} denotes the number of coalitions sampled ($M$) for estimating the expected marginal contribution of each token (via segments) under the Owen value framework (\Cref{alg:ospo}). We additionally include an \texttt{All Subsets} variant that removes the contiguity constraint (Line 5 in \Cref{alg:ospo})  by sampling from “all” possible non-contiguous coalitions up to $w_{\text{max}}$, allowing us to assess how the loss of local structure and linguistic contiguity affect marginal contribution estimation.

\begin{table}[t]
\centering
\small
\setlength{\tabcolsep}{3.2pt}
\renewcommand{\arraystretch}{0.9}
\begin{tabular}{lcccc}
\toprule
\textbf{Ablation} & \textbf{NDCG} & \textbf{Recall} & \textbf{MRR} & \textbf{AP} \\
\midrule
\midrule
\multicolumn{5}{c}{\textbf{num\_generations = 8}}\\
\cmidrule(lr){1-5}
OSPO-w8-p96         & 0.7821 & 0.8375 & 0.7625 & 0.7625 \\
OSPO-w16-p256       & 0.7637 & 0.8375 & 0.7375 & 0.7375 \\
OSPO-w2-p64         & 0.7558 & 0.8250 & 0.7312 & 0.7312 \\
OSPO-w3-p24         & 0.7545 & 0.8375 & 0.7250 & 0.7250 \\
OSPO-w6-p32         & 0.7525 & 0.8125 & 0.7312 & 0.7312 \\
OSPO-w1-p16         & 0.7183 & 0.7875 & 0.6937 & 0.6937 \\
OSPO-w4-p48         & 0.7058 & 0.7750 & 0.6813 & 0.6813 \\
OSPO-w12-p128       & 0.5946 & 0.6500 & 0.5750 & 0.5750 \\
\midrule
\multicolumn{5}{c}{\textbf{num\_generations = 2}}\\
\cmidrule(lr){1-5}
OSPO-w8-p96    & 0.2500 & 0.2500 & 0.2500 & 0.2500 \\
OSPO-w2-p64   & 0.2375 & 0.2375 & 0.2375 & 0.2375 \\
OSPO-w1-p16   & 0.2375 & 0.2375 & 0.2375 & 0.2375 \\
OSPO-w12-p128 & 0.2250 & 0.2250 & 0.2250 & 0.2250 \\
\midrule
\multicolumn{5}{c}{\textbf{All Subsets}}\\
\cmidrule(lr){1-5}
OSPO-w4-p8          & 0.1625 & 0.1625 & 0.1625 & 0.1625 \\
OSPO-w4-p16         & 0.1329 & 0.1375 & 0.1313 & 0.1313 \\
OSPO-w8-p16         & 0.1125 & 0.1125 & 0.1125 & 0.1125 \\
OSPO-w8-p8          & 0.0408 & 0.0500 & 0.0375 & 0.0375 \\
\bottomrule
\end{tabular}
\caption{Training-time ablation results on coalition structure on the ESCI product search benchmark for \textsc{OSPO-Prop} variants. Each entry reports the mean performance across 10 evaluations (every 100 steps of training) conducted on batch-level samples from 100 randomly sampled queries from the ESCI evaluation set. Here, \textbf{w} denotes the maximum coalition span ($w_{max}$ and \textbf{p} the number of coalitions sampled ($M$) for Owen value estimation in OSPO (\Cref{alg:ospo}). Note that MRR and AP converge in this batch-level monitoring setup due to the predominance of single-relevant-item queries in evaluation batches. }

\label{tab:ospo_ablation_esci}
\end{table}

\paragraph{Coalition Structure and Exploration as Drivers of Stable and Generalizable Learning} As summarized in Table~\ref{tab:ospo_ablation_esci}, OSPO’s performance in the conventional product search task on ESCI shopping queries is governed by the interaction between \emph{coalition structure} and \emph{exploration}. Moderate, contiguous configurations (e.g., \texttt{w8\_p96}, \texttt{w6\_p32}) consistently yield the strongest NDCG/MRR—balancing local compositional context with sufficient permutation diversity—whereas very narrow coalitions (e.g., \texttt{w1}, \texttt{w2}) overfit to local co-occurrences, and very wide coalitions (e.g., \texttt{w12}, \texttt{w16}) underutilize the sampling budget given typical query lengths. \texttt{All Subsets} variants underperform because they discard contiguity and disperse a fixed sampling budget over an exponentially larger non-contiguous space, producing high-variance Owen estimates and a weaker token-level credit signal. Reducing Monte Carlo generations further degrades scores across the board, underscoring that exploration via REINFORCE-based~\citep{williams92, shao2024deepseekmath} multiple completions is indispensable for stable advantage redistribution in OSPO. These trends mirror our H\&M fashion results for contextualized product search task—best generalization arises from \emph{contiguous, moderate-width coalitions with adequate permutations and generation sampling}, rather than from either extremely narrow/wide spans or non-contiguous subset sampling.

\begin{table}[t]
\centering
\small
\begin{tabular}{lcccc}
\toprule
\textbf{Method} & \textbf{Final NDCG} & \textbf{Mean NDCG} & \textbf{Std NDCG} & \textbf{Max NDCG} \\
\midrule
w4\_p48   & 0.625000 & 0.688851 & 0.114355 & 0.972842 \\
w3\_p24   & 0.609375 & 0.689658 & 0.113945 & 0.968750 \\
w16\_p256 & 0.593750 & 0.688587 & 0.111952 & 1.000000 \\
w2\_p64   & 0.587983 & 0.687731 & 0.109890 & 0.953125 \\
w12\_p128 & 0.578125 & 0.686309 & 0.111869 & 0.968750 \\
w8\_p96   & 0.562500 & 0.687286 & 0.111027 & 0.953125 \\
w6\_p32   & 0.562500 & 0.686912 & 0.112748 & 0.962983 \\
w1\_p16   & 0.546875 & 0.691219 & 0.110675 & 0.957217 \\
\bottomrule
\end{tabular}
\caption{Training performance (NDCG) on the H\&M contextualized product search dataset. Values reported at the final checkpoint \textbf{step = 1000}; we also include mean, standard deviation, and maximum NDCG observed over training. Higher is better.}
\label{tab:hnm_coalition_ablation_train}
\end{table}

\begin{figure}[H]
  \centering
\includegraphics[width=0.85\linewidth]{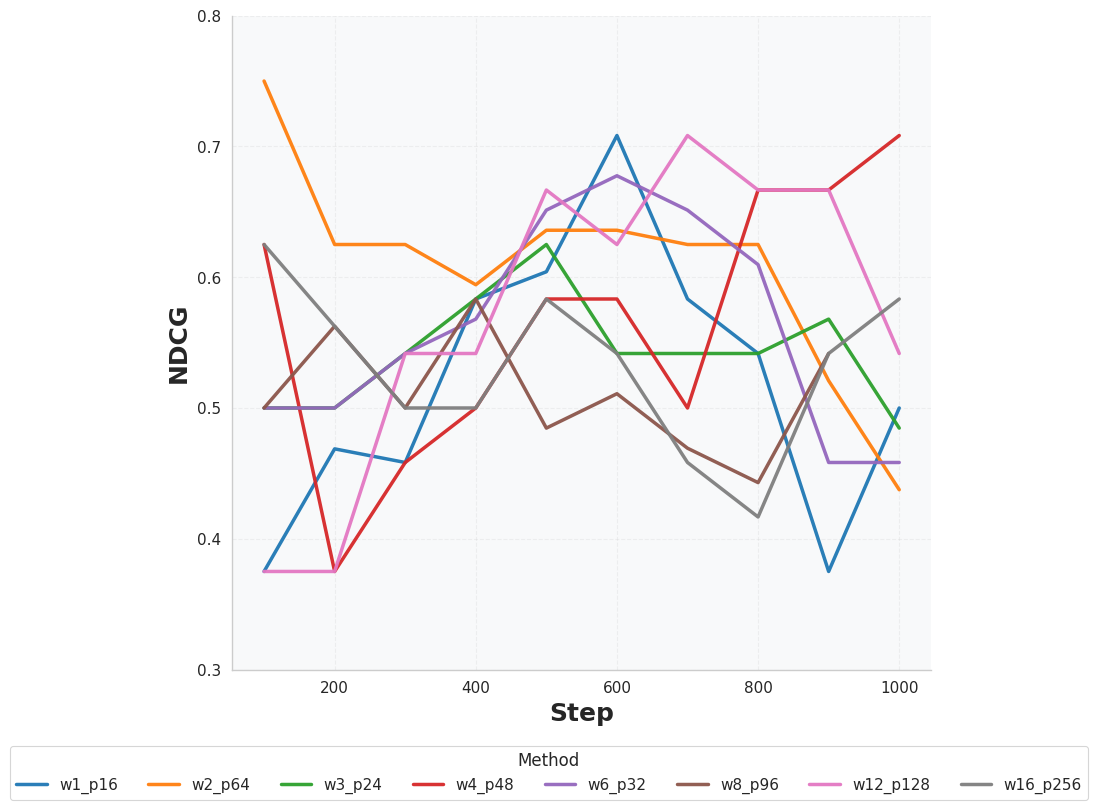}

\caption{
\textbf{Test NDCG on 100 H\&M contextualized queries across coalition strategies.}
Each curve corresponds to a fixed coalition width–permutation configuration. We vary the \textbf{maximum span width} $W \in \{1,2,3,4,6,8,12,16\}$ and the \textbf{maximum permutation count} $p \in \{16,24,32,48,64,96,128,256\}$.
Medium-width configurations (e.g., \texttt{w4\_p48}, \texttt{w6\_p32}, \texttt{w8\_p96}) yield the most reliable evaluation trajectories and highest final NDCG (up to $\approx0.71$), while very small widths oscillate or degrade and large $p$ values show diminishing returns given most contextualized queries are short with limited phrase diversity/domain-specificity.
Here, \textbf{w} denotes the maximum coalition span ($w_{max}$) and \textbf{p} the number of coalitions sampled ($M$) for Owen value estimation in OSPO (\Cref{alg:ospo}).
These results indicate that coalition design mainly influences \emph{generalization} rather than serving as a simple optimization hyperparameter.
}

  \label{fig:hnm_coalition_ablation}
  \vspace{-2mm}
\end{figure}

\vspace{-5mm}

\subsection{Coalition Width and Learning Dynamics}
\label{ssec:coalition_learning_dynamics}
\Cref{fig:hnm_coalition_ablation} presents an ablation of coalition design in OSPO on \textbf{100 held-out H\&M contextualized queries}, where we vary the \textbf{maximum coalition span width} 
\(W \in \{1,2,3,4,6,8,12,16\}\) and the \textbf{permutation budget} 
\(p \in \{16,24,32,48,64,96,128,256\}\), and track test NDCG throughout training.
Each curve corresponds to a fixed \((W,p)\) configuration and reflects its generalization behavior over time.

The results reveal a clear \textbf{non-monotonic relationship} between coalition width and generalization performance.
Very narrow coalitions (\(W \leq 2\)) exhibit rapid early gains but suffer from high variance and pronounced late-stage degradation, indicating unstable credit assignment when marginal contributions are estimated with insufficient context. At the other extreme, very wide coalitions (\(W=16\)) fail to outperform moderate settings, as typical queries contain only 6--10 segmented phrases, severely limiting the number of realizable coalitions and rendering large permutation budgets ineffective.
Increasing the permutation count \(p\) beyond the achievable coalition diversity does not improve estimation quality and instead leads to redundant resampling. In contrast, moderate-width coalitions—particularly \(W=4\) with a moderate permutation budget (\texttt{w4\_p48})—exhibit delayed but stable convergence, achieving the strongest final test NDCG with low variance and no late-stage collapse.
These findings suggest that coalition width acts as an “implicit” regularizer in OSPO, mediating a trade-off between noisy local credit assignment and underutilized compositional structure.
Below, we provide a detailed analysis of the learning dynamics and failure modes of each configuration.

\subsubsection{Small Coalitions: Fast Learning, Early Collapse}

Narrow coalition widths (w2\_p64) exhibit rapid initial learning, achieving NDCG~0.75 by step 200, but subsequently degrade substantially to 0.44 by step 1000. This pattern reflects a fundamental trade-off: smaller coalitions converge quickly by identifying surface-level patterns (e.g., ``black jeans'' or ``high waist'' as isolated features), but lack the compositional structure needed to capture how multiple preference attributes interact. 

The degradation after step 600 suggests that width-2 coalitions introduce noise into the credit assignment signal. When evaluating coalitions containing only 1--2 consecutive phrases, marginal contribution estimates become unstable—a single phrase's Owen value fluctuates depending on which adjacent phrase appears in the coalition, without sufficient context to disambiguate genuine causal effects from spurious correlations. This instability causes the model to oscillate between competing hypotheses about which tokens drive performance, ultimately converging to a suboptimal solution that fails to generalize.

\subsubsection{Very Wide Coalitions: Underutilization Due to Query Length}

Configurations with maximum width (w16\_p256) fail to outperform more moderate settings despite their theoretical expressiveness. Analysis of generated queries during training reveals the constraint: \textbf{most queries contain only 6--10 segmented phrases}, making width-16 coalitions impossible to construct. Even when the maximum permutation limit is set to 256, we observe that only approximately 41--48 unique coalitions are actually sampled per query, as the contiguous constraint combined with limited phrase count restricts the coalition space.

This underutilization has two consequences. First, the model cannot leverage the fine-grained credit assignment that wider coalitions theoretically enable due to the lack of enough phrases to form them. Second, the high permutation budget (256) is wasted on resampling similar small coalitions, rather than exploring diverse coalition structures. As a result, w16\_p256 performs no better than more efficient configurations with moderate widths, demonstrating that coalition strategy must account for the \textit{actual} query structure distribution, not just the maximum possible width.

\subsubsection{Moderate-Wide Coalitions: Delayed but Stable Convergence}

The w12\_p128 configuration demonstrates a distinct learning profile: slower initial learning compared to small coalitions, but convergence to a stable solution with sustained performance. This configuration peaks around step 700 (NDCG~0.71) and maintains reasonable performance through step 1000, suggesting that width-12 coalitions provide sufficient compositional range for the task without exceeding typical query lengths.

However, w12\_p128 exhibits late-training instability (declining from 0.71 to 0.55 between steps 700--1000), indicating potential overfitting to the coalition structure itself. With 128 sampled coalitions per query, the model may learn spurious correlations specific to the training coalition distribution rather than robust compositional patterns.

\subsubsection{Optimal Range: w4\_p48 as the Goldilocks Configuration}

The configuration with width $W=4$ and $P=48$ permutations achieves the best final performance (NDCG~0.71) through a qualitatively different learning trajectory. Unlike narrow coalitions that peak early or wide coalitions that show volatility, w4\_p48 exhibits \textbf{gradual, monotonic improvement} throughout training, with its steepest gains occurring in the final 400 steps (0.63 at step 600 $\rightarrow$ 0.71 at step 1000).

This delayed convergence reflects a more principled learning process. Width-4 coalitions can capture meaningful multi-attribute patterns (e.g., ``black skinny jeans + stiff fabric'') without requiring query lengths that rarely occur. The moderate permutation count (48) provides sufficient diversity for robust Owen value estimation without oversampling, avoiding both the noise of w2\_p64 and the overfitting risk of w12\_p128.

Critically, w4\_p48's continued improvement in late training—when other configurations plateau or degrade—suggests it does not exhibit early convergence or late-stage collapse. Instead, the gradual ascent indicates the model is progressively refining its understanding of compositional query semantics, leading to solutions that generalize to novel test cases.

\subsubsection{The Role of Coalition Noise}

The dramatic failure of w2\_p64 (collapsing from 0.75 to 0.44) and instability of w1\_p16 validate a key intuition: \textbf{coalition width acts as a regularization parameter}. Excessively narrow coalitions introduce noise by providing insufficient context for stable marginal contribution estimates. When Owen values are computed from coalitions containing only 1--2 tokens, the credit signal becomes dominated by local interactions rather than compositional semantics.

This noise has a delayed effect—early training benefits from rapid exploitation of surface patterns, but as training progresses, the accumulation of noisy credit assignments leads the model to overweight spurious features. By contrast, width-4 coalitions filter this noise by requiring evaluation contexts large enough to disambiguate genuine compositional effects from random co-occurrences, leading to more stable and generalizable learning.

\subsection{Coalition Ablations with a Smaller Backbone Model}
\label{ssec:coalition_learning_dynamics_smaller}
To complement the coalition ablation results obtained with 7B-scale models, we repeat the analysis using a significantly smaller backbone, \textbf{Qwen~2.5--1.5B--Instruct}, in order to assess the robustness of OSPO’s coalition dynamics under reduced model capacity.
This setting allows us to efficiently explore a broad range of coalition widths (\(w\)) and sampling counts (\(p\)), while also enabling finer-grained diagnostics that are more difficult to isolate at larger scales, such as reward sparsity and patterns, and potential reward hacking.
Due to computational constraints, we down-sample the contextualized H\&M search training data to 1{,}000 samples and evaluate performance after a single training epoch on 100 held-out queries.
Results are summarized in Table~\ref{tab:coalition-ablations}.
Importantly, while absolute retrieval performance is lower than in the 7B experiments, the qualitative trends in coalition behavior remain highly consistent, providing additional evidence that OSPO’s coalition design principles are stable across model scales.

\begin{table}[t]
\centering
\small
\setlength{\tabcolsep}{6pt}
\begin{tabular}{lcccc}
\toprule
\textbf{Configuration} & \textbf{Max Width} & \textbf{NDCG} & \textbf{Avg.~Tokens} & \textbf{Zero~Rate} \\
\midrule
w1\_p16    & 1  & 0.222 & 145.6  & 43.8\% \\
w2\_p32    & 2  & 0.264 & 173.7  & 50.0\% \\
w2\_p64    & 2  & 0.332 & 165.9  & 12.5\% \\
w3\_p48    & 3  & 0.237 & 202.5  & 31.2\% \\
w4\_p24    & 4  & \textbf{0.360} & 191.7  & \textbf{6.2\%} \\
w6\_p64    & 6  & \textbf{0.360} & 149.3  & 12.5\% \\
w8\_p96    & 8  & 0.300 & 199.1  & 40.6\% \\
w12\_p128  & 12 & 0.349 & 157.0  & 18.8\% \\
w12\_p256  & 12 & 0.271 & 132.1  & 43.8\% \\
\bottomrule
\end{tabular}
\caption{Owen--Shapley coalition ablation results after one epoch on H\&M contextualized product search data, evaluated on the training set using 1{,}000 samples for efficiency. 
Experiments use the Qwen~2.5--1.5B--Instruct model. 
Here, \textbf{w} denotes the maximum coalition span (consecutive phrases) and \textbf{p} the number of coalitions sampled for Owen value estimation. 
In the OSPO algorithm (\Cref{alg:ospo}), \textbf{w} corresponds to the maximum coalition width \(w_{\max}\) and \textbf{p} to the total number of sampled coalitions \(M\). 
NDCG@1000 measures ranking quality on pooled retrieval candidates, and \textbf{Zero~Rate} denotes the proportion of completions with zero reward. 
Best values are highlighted in bold.}
\label{tab:coalition-ablations}
\end{table}

\paragraph{Coalition width (\(w\))}
As shown in Table~\ref{tab:coalition-ablations}, moderate widths (\(w{=}4\text{--}6\)) yield the highest ranking quality (\textsc{NDCG}~\(\approx 0.36\)), while both narrow (\(w{\le}2\)) and wide (\(w{\ge}12\)) configurations underperform (\textsc{NDCG}~\(\le 0.33\)).
This reflects a trade-off between contextual coverage and noise accumulation.
Wider coalitions concatenate more extracted phrases, increasing the chance of mixing meaningful segments (e.g., ``comfortable jeans'', ``breathable linen'') with fragmentary artifacts (e.g., ``-~wash'', ``a~step'').
This arises from the \texttt{spaCy}\footnote{\url{https://spacy.io/models/en}} phrase extraction pipeline, which can occasionally capture short or incomplete fragments when forming noun or verb-led spans.
When such noisy fragments are concatenated under large coalition widths, they distort semantic coherence and yield incoherent queries (e.g., ``comfortable jeans a weekend outing -~wash mid wash'').
As a result, Owen--Shapley credit assignment distributes marginal contributions across irrelevant tokens, weakening policy updates.
Moderate-width coalitions, by contrast, preserve compositional context while suppressing spurious phrase interactions, producing clearer and more focused learning signals.

\textbf{Coalition sampling count (\(p\))}
The number of sampled coalitions controls the variance of the Owen value approximation. 
Moderate sampling densities (\(p{=}24\text{--}64\)) implicitly regularize training by limiting exposure to noisy phrase combinations, yielding both high NDCG and low zero-reward rates (6--12\%). 
Excessively large \(p\) values (\(p{=}256\)) degrade ranking quality (NDCG~0.27) and increase the fraction of zero-reward samples (43.8\%), 
as exhaustive coalition coverage overfits to irrelevant phrase subsets and disperses credit too broadly. 
This indicates that stochastic coalition sampling not only reduces computational cost but also stabilizes the reward signal during policy optimization.

\textbf{Assessing reward hacking}
To investigate whether \ospoProp{} exhibits reward hacking through query length manipulation, we compute the Pearson correlation coefficient between query length (in tokens) and reward across all configurations in \Cref{tab:coalition-ablations}. The observed correlation of $-0.165$ indicates a weak negative relationship, meaning longer queries tend to perform slightly \textit{worse} rather than better. 
This finding provides little evidence of reward hacking—if models were exploiting length to maximize rewards~\citep{skalse2025definingcharacterizingrewardhacking}, we would observe a strong positive correlation instead. 
The weakness of this correlation (close to zero) further suggests that query length is not a primary driver of retrieval performance; rather, semantic coherence and phrase quality matter more for effective search query generation under \ospoProp{}.

\textbf{Coalition Reward Patterns}
To understand how coalition composition affects retrieval performance, we examine individual coalition-level rewards during training.
Successful coalitions (non-zero rewards) consistently combine three key elements: concrete product identifiers (e.g., jeans'', coat'', dress''), specific attributes (e.g., mid-blue wash'', high-waist'', ankle length''), and customer intent or use-case context (e.g., weekend outing'', beach vacation'').
For example, the coalition \textit{your current wardrobe various situations cooler spring days a rich blue hue shorter length neutral colors a coat''} achieves perfect reward (1.0) by balancing contextual framing with discriminative product features. In contrast, coalitions receiving zero reward fall into predictable failure modes: generic phrases without product anchors (a pair'', modern fit''), pure adjective sequences lacking concrete nouns (sleek polished clean''), junk phrase contamination from parsing artifacts (- wash'', a step''), or excessive repetition and incomplete fragments (few lightweight... lightweight... a touch'').  Notably, even coalitions containing strong individual phrases fail when dominated by vague descriptors; for instance, \textit{prefer a vibrant color to add my usual neutral tones my beach vacation some personality''} receives zero reward despite capturing user preferences, because it lacks a specific product type or actionable attributes.
This analysis suggests that moderate coalition widths ((w{=}4\text{--}6)) succeed by maintaining sufficient phrase diversity to construct discriminative queries while avoiding the accumulation of noisy fragments that degrade semantic coherence in wider configurations.
 
Together, these findings suggest that \textbf{\ospoProp{} achieves optimal balance at moderate coalition configurations} (\(w{=}4\text{--}6,~p{=}24\text{--}64\)), 
where the phrase context is sufficiently rich to capture meaningful dependencies but constrained enough to avoid semantic drift. 
The corresponding low zero-reward rates and stable token lengths further confirm that controlled coalition sampling preserves dense and informative gradients, leading to robust policy learning and higher downstream ranking quality.

\begin{table*}[t]
\centering
\small
\begin{tabular}{p{0.95\linewidth}}
\toprule
\textbf{Prompt and Completion} \\
\midrule

\textbf{System:} You are a personal fashion advisor and stylist. Analyze the shopping history and rank \textbf{all provided candidate items} from most to least recommended for the next purchase. Consider style evolution, wardrobe needs, and purchasing patterns. Respond using \texttt{<thinking>} [analysis] \texttt{</thinking>} and \texttt{<ranking>} [ordered item IDs] \texttt{</ranking>}. \\

\addlinespace
\textbf{User:} 25 years old, inactive club member, no fashion news subscription. \newline

\textbf{Purchase history (summarized):} sweaters, leggings, bodysuits, trousers, and tops (2018–2019), primarily in pastel or neutral tones and low price range (\$0.01–\$0.02).

\textbf{Candidate items:} a mix of garments including swimsuits, a garment set, bodysuits, caps, bikini tops, trousers, blazers, dresses, and tops. Example entries include: \newline
1. [742933003]: C Jackpot Swimsuit – Swimsuit in Black (\$0.03) \newline
2. [767228001]: FASHION OFFER GRAYHAM SET – Garment Set in Off White (\$0.02) \newline
3. [630416003]: Timber L/S Body – Bodysuit in Red (\$0.01) \newline
... \newline
9. [895993004]: Dag Tank – Vest top in Light Beige (\$0.01) \newline
10. [822180002]: Tula – Sweater in Black (\$0.03) \\

\textbf{Assistant:} \newline
\texttt{<thinking>} \newline
Analyzing the customer's purchase history, it's clear that they have strong preferences for comfortable, modest-priced items and that they frequently purchase basic pieces like sweaters, undergarments, and simple tops. Each purchase has been around the \$0.01 mark, indicating a budget-conscious shopper. They have also purchased a variety of colors with an emphasis on pastels and neutral tones. Given their current inactivity and club member role, they might be looking for a more stylish yet comfortable piece that can be worn for various occasions, including social events. \newline
\newline
Here’s my analysis: The swimwear (item 1) and sporty top (item 5) are not suitable as purchases since there is no mention of beach trips or fitness activities. The blazer (item 7) is expensive and more formal, which does not align with their purchase history and budget. The high-end dress (item 8) would be a significant jump in price and might be out of their comfort zone, considering their current dresses were all within the \$0.01 price range. The expensive sweaters (items 10 and 3) are not suitable given their consistent purchase price. Items 2, 6, 9, and 4 (FASHION OFFER GRAYHAM SET, Lilja trousers, Dag Tank, and KELLY basic P-CAP) are all within the \$0.02–\$0.03 price range and offer different styles that fit well with their wardrobe. Considering the current fashion and the customer's budget and preferences, I recommend a versatile item with a neutral color and casual-chic look. \newline
\newline
\texttt{<ranking>} [767228001, 895993004, 790057001, 811907005] \texttt{</ranking>} \\
\bottomrule
\end{tabular}
\caption{Table showing an illustrative user summary generated by the \textsc{OSPO-prop}–trained Qwen2.5-7B-Instruct model in the user summary generation task (main results in~\Cref{tab:ospo_wr_compact}). System and user prompts are trimmed for brevity.
}
\label{fig:hnm_summary_fullprompt}
\end{table*}

\begin{table*}[t]
\centering
\small

\begin{tabular}{p{0.03\linewidth} p{0.93\linewidth}}
\toprule
\textbf{S\#} & \textbf{Sentence Text} \\
\midrule
S0 & Analyzing the customer's purchase history, it's clear that they have strong preferences for comfortable, modest-priced items and that they frequently purchase basic pieces like sweaters, undergarments, and simple tops. \\

S1 & Each purchase has been around the \$0.01 mark, indicating a budget-conscious shopper. \\

S2 & They have also purchased a variety of colors with an emphasis on pastels and neutral tones. \\

S3 & Given their current inactivity and club member role, they might be looking for a more stylish yet comfortable piece that can be worn for various occasions, including social events. \\

S4 & Here’s my analysis: \\

S5 & The swimwear (item 1) and sporty top (item 5) are not suitable as purchases since there is no mention of beach trips or fitness activities. \\

S6 & The blazer (item 7) is expensive and more formal, which does not align with their purchase history and budget. \\

S7 & The high-end dress (item 8) would be a significant jump in price and might be out of their comfort zone, considering their current dresses were all within the \$0.01 price range. \\

S8 & The expensive sweaters (items 10 and 3) are not suitable given their consistent purchase price. \\

S9 & Items 2, 6, 9, and 4 (FASHION OFFER GRAYHAM SET, Lilja trousers, Dag Tank, and KELLY basic P-CAP) are all within the \$0.02–\$0.03 price range, and offer different styles that can fit well with their wardrobe. \\

S10 & Considering the current fashion and the customer's budget and preferences, I would recommend a versatile item that can fit most scenarios, thus likely a piece with a more neutral color and a casual yet chic look. \\

S11 & \texttt{<ranking> [767228001, 895993004, 790057001, 811907005] </ranking>} \\
\bottomrule
\end{tabular}
\caption{OSPO User summary generation sample: Sentence-level reasoning breakdown from assistant completion. Each sentence corresponds to $S_i$ in the coalition analysis.}
\label{tab:ospo_user_summary_sentence_breakdown}
\end{table*}

\begin{table*}[t]
\centering
\small
\begin{tabular}{r l r p{0.55\linewidth}}
\toprule
\textbf{ID} & \textbf{Coalitions} & \textbf{Reward} & \textbf{Interpretation} \\
\midrule
0  & [] & $-3.2290$ & Empty user summary — baseline negative reward; \\

1  & [0] & $3.6821$ & First sentence alone carries strong global context (“preferences for comfortable, modest-priced items”). Establishes key signal. \\

2  & [0,1] & $4.4063$ & Early purchase-pattern reasoning (budget inference) compounds positively — synergy between S0 and S1. \\

3  & [0,1,2] & $4.7512$ & Adds color palette reasoning (S2), maintaining strong coherence; reward increases marginally. \\

4  & [0,1,2,3] & $3.4214$ & Adding lifestyle speculation (S3) introduces slight noise; reasoning starts to deviate. \\

5  & [0,1,2,3,4] & $4.7757$ & Meta-structure (“Here’s my analysis:”) helps logical segmentation — reward briefly peaks. \\

6  & [0–5] & $-2.9897$ & Insertion of item-level exclusions (S5) sharply reduces relevance; excessive pruning or bias. \\

7  & [0–6] & $0.3499$ & Reintroduces structured filtering, partial recovery but still weak. \\

8  & [0–7] & $-1.0558$ & Over-constraint: adds formalwear exclusion (S7), possibly too harsh for general ranking. \\

9  & [0–8] & $-0.8890$ & Continues negative trend — overly restrictive logic misaligns with ground-truth ranking. \\

10 & [0–9] & $3.9120$ & Large recovery — includes specific candidate discussion (S9), capturing useful selection rationale. 
Highest observed reward before verbosity penalty. \\

11 & [0–10] & $3.0907$ & Adding generic stylistic recommendation (S10) slightly dilutes retrieval precision. \\

12 & [0–11] & $2.6979$ & End-of-trace expansion (ranking emission) inflates token count; diminishing returns. \\

13 & [1–11] & $1.8037$ & Excluding S0 removes global framing → coherence loss, partial reward drop. \\

14 & [2–11] & $2.1026$ & Partial recovery via descriptive attributes (S2), but weak without initial context. \\
\bottomrule
\end{tabular}
\caption{Coalition formation and rewards for sampled coalitions in the user profile summary generation task, where sentences act as ``players'' (segments). Each coalition's reward is computed from sentence segments extracted from Chain-of-Thought traces of the Qwen2.5-7B-Instruct base model. The reward estimator is the same base model with a trained linear head on expert-generated contrastive user profile samples. Segment details are provided in \Cref{tab:ospo_user_summary_sentence_breakdown}.}
\end{table*}

\begin{table*}[t]
\centering
\small
\begin{tabular}{r l r p{0.55\linewidth}}
\toprule
\textbf{ID} & \textbf{Sentence Subset} & \textbf{Reward} & \textbf{Interpretation} \\
\midrule

15 & [3–11] & $1.8825$ & Retains lifestyle speculation but loses profile grounding; lower reward. \\
16 & [4–11] & $1.9275$ & Pure analysis section without framing; moderately neutral. \\

17 & [5–11] & $2.5463$ & Item-filtering logic dominates — fair local precision, low diversity. \\

18 & [6–11] & $2.5095$ & Similar to (17); mild degradation, redundant exclusions. \\

19 & [7–11] & $3.1307$ & Focus on mid/late reasoning; surprising bump — possibly due to item-specific reasoning (S9). \\

20 & [8–11] & $2.1938$ & Partial sentence window (S8–S11) lacks structure; reward decays. \\

21 & [9–11] & $2.9035$ & Captures shortlist discussion (S9–S11) — relevant but lacks justification context. \\

22 & [10,11] & $1.3301$ & Contains generic stylistic advice and ranking output; relevant but low semantic grounding. \\

23 & [11] & $0.5208$ & Ranking-only output; minimal information — weak semantic signal. \\

24 & [1,2] & $3.8981$ & Crucial early coalition: purchase habits + color reasoning; high consistency and relevance. \\

25 & [2,3] & $3.0590$ & Local coalition of color + lifestyle; moderate contextual fit. \\

26 & [3,4] & $3.8642$ & Transition from lifestyle to analysis marker — meaningful logical flow, strong local coherence. \\

27 & [4,5] & $0.0054$ & Weak handoff; introduction of filtering abruptly reduces clarity. \\

28 & [5,6] & $2.4154$ & Filtering criteria + pricing logic; regains some precision. \\

29 & [6,7] & — & (Not logged) Likely small negative or neutral value given surrounding trend. \\

30 & [7,8] & $-0.4614$ & Consecutive exclusions (formal/dress) yield overfitting — retrieval misalignment. \\

31 & [8,9] & $0.0060$ & Weak positive recovery; transition toward item justification. \\

32 & [9,10] & $-0.7675$ & Adding stylistic advice to item shortlist confuses ranking criteria; drop in reward. \\
\bottomrule
\end{tabular}
\caption{OSPO Rewards (Continued): Coalition combinations, rewards, and interpretive comments. 
Rewards correspond to partial-sentence subsets of the expert reasoning trace.}
\label{tab:user_summary_examples_last_table}
\end{table*}

\begin{table*}[t]
\centering
\small
\setlength{\tabcolsep}{5pt}
\renewcommand{\arraystretch}{1.1}
\begin{tabular}{p{0.16\linewidth} p{0.40\linewidth} p{0.40\linewidth}}
\toprule
\textbf{Field} & \textbf{GRPO} & \textbf{\ospoProp{}} \\
\midrule
\multicolumn{3}{l}{\textit{Example 1 — Latent intent shift}} \\
\midrule
\textbf{User Query} &
\multicolumn{2}{p{0.80\linewidth}}{
I have this amazing oversized black sweater that's my cozy weekend uniform with leggings. But I'd love to mix it up with something a little lighter and brighter for spring. A slouchy, relaxed yellow top could be the perfect swap — still laid-back but refreshing.
} \\
\textbf{Purchase History} &
\multicolumn{2}{p{0.80\linewidth}}{
Sacramento dress (Black), ROSIE HW Biker leggings (Black), Neve Off Shoulder sweater (Black), Fidde tee (White), Mermaid Swimsuit (Black), Lana seamless bralette (Black), Porsche dress (Black).
} \\
\textbf{Target Item} &
\multicolumn{2}{p{0.80\linewidth}}{
\texttt{892100001} — Lightweight, relaxed sweater with breathable construction and brighter styling for transitional wear.
} \\
\midrule
\textbf{Expanded Query} &
Discover our vibrant collection of slouchy tops! Crafted from high-quality fabrics like cotton-blend and rayon, these pieces provide comfort without compromising on style. Ideal for warmer months, each piece adds a touch of freshness to your wardrobe. &
This slouchy, relaxed yellow top would perfectly transition from the cozy comfort of their oversized black sweater to the warmer weather of spring. It maintains a laid-back vibe while adding a touch of brightness and freshness. The yellow color will complement the black sweater nicely. \\
\textbf{Target in Top-1000} & No & \textbf{Yes} \\
\textbf{NDCG} & 0.0000 & \textbf{1.0000} \\
\midrule
\multicolumn{3}{l}{\textit{Example 2 — Purchase history hallucination}} \\
\midrule
\textbf{User Query} &
\multicolumn{2}{p{0.80\linewidth}}{
I've been living in my black trousers lately, but they can feel a bit heavy and constricting when the weather warms up. I'm looking for something breezier and more casual that still has that sleek, tailored vibe. Maybe a lightweight short with a relaxed fit?
} \\
\textbf{Purchase History} &
\multicolumn{2}{p{0.80\linewidth}}{
Reese sneakers (Light Beige), Maria dress (Green), Chaka linen T-shirt (Light Beige), Sun dress (Black), Lyra HW wide full trousers (Black), Queen Sweater (White).
} \\
\textbf{Target Item} &
\multicolumn{2}{p{0.80\linewidth}}{
\texttt{817472001} — Softly draping viscose shorts, breathable comfort for warm weather, high-waisted tailored silhouette.
} \\
\midrule
\textbf{Expanded Query} &
This summer, consider the Lyra Wide Full Trousers in Light Beige. With a relaxed yet tailored design, these pants offer a breathable, casual look perfect for warm days. &
Discover our new collection of breezy shorts! Crafted with a relaxed fit and a sleek, tailored look, these shorts offer a perfect balance between comfort and style, maintaining the same sophisticated silhouette as existing black trousers. \\
\textbf{Target in Top-1000} & No & \textbf{Yes} \\
\textbf{NDCG} & 0.0000 & \textbf{0.8333} \\
\textbf{Top-5 Retrieved} &
\texttt{546565010} &
\texttt{817472001} (target), \texttt{879307002}, \texttt{739144008}, \texttt{772711004}, \texttt{781813005} \\
\bottomrule
\end{tabular}
\caption{Qualitative comparison of GRPO and \ospoProp{} on contextualized H\&M product search.}
\label{tab:grpo_ospo_case_study}
\end{table*}

\begin{table*}[t]
\centering
\small
\setlength{\tabcolsep}{5pt}
\renewcommand{\arraystretch}{1.1}
\begin{tabular}{p{0.16\linewidth} p{0.40\linewidth} p{0.40\linewidth}}
\toprule
\textbf{Field} & \textbf{GRPO} & \textbf{\ospoProp{}} \\
\midrule
\multicolumn{3}{l}{\textit{Example 3 — Color preference inference}} \\
\midrule
\textbf{User Query} &
\multicolumn{2}{p{0.80\linewidth}}{
My usual black and white tees are starting to feel a little boring to me. I'd love to add a pop of color to brighten up some of my outfits, but I still want that same soft, casual vibe. Maybe a pretty pastel or soft pink top that I could wear with my black trousers or under my denim shirt?
} \\
\textbf{Purchase History} &
\multicolumn{2}{p{0.80\linewidth}}{
Madison Slim Stretch Chino (Light Beige), Oregano dress (Black), Bon HW straight trousers (White), Jentry Denim Shirt (Black), Cat Tee (White), Blenda tee (Black).
} \\
\textbf{Target Item} &
\multicolumn{2}{p{0.80\linewidth}}{
\texttt{880312005} — The Delila ss cropped tee offers a youthful warm-weather option in a pastel shade, complementing casual separates.
} \\
\midrule
\textbf{Expanded Query} &
Add a bold pastel or soft pink t-shirt featuring a coordinated design with existing wardrobe staples: light beige chinos, black oregano dress, white straight trousers, black sandals, off-white wedding dress, and classic black tee. &
This pastel or soft pink t-shirt would complement Madison's slim stretch chinos and her black dress. It maintains the casual vibe while adding a pop of color. The soft texture aligns well with her preference for comfort and casual attire. \\
\textbf{Target in Top-1000} & No & \textbf{Yes} \\
\textbf{NDCG} & 0.0000 & \textbf{0.6667} \\
\textbf{Top-5 Retrieved} &
— &
\texttt{880312005} (target): Delila ss cropped tee, pastel shade \\
\midrule
\multicolumn{3}{l}{\textit{Example 4 — Style attribute transfer}} \\
\midrule
\textbf{User Query} &
\multicolumn{2}{p{0.80\linewidth}}{
I'm really digging the cropped sweater trend for fall, but I want to find one that's not too short or tight. My white Jennifer blazer has such a nice boxy, relaxed fit, so I'm wondering if there are any sweaters out there with a similar vibe? Neutral colors would be ideal to mix and match with my existing pieces.
} \\
\textbf{Purchase History} &
\multicolumn{2}{p{0.80\linewidth}}{
Ravello shackett (Off White), Mutt car coat (Beige), Jennifer Blazer (White), Tori dress (Black), Wowcha top (Black).
} \\
\textbf{Target Item} &
\multicolumn{2}{p{0.80\linewidth}}{
\texttt{889753004} — The Fazer balloon sweater has an of-the-moment cropped, boxy shape in a fresh neutral shade.
} \\
\midrule
\textbf{Expanded Query} &
Discover a cozy, boxy, neutral sweater alternative crafted in soft fabric with a relaxed silhouette. Its flattering boxy shape draws inspiration from the classic Jennifer blazer. Complements existing blazers and chic swimwear. &
Discover a versatile cropped sweater with a boxy, relaxed fit reminiscent of the Jennifer blazer. Perfectly neutral in color, coordinates with existing white Ravello shackett, black Tori dress, and dark green Wowcha top. \\
\textbf{Target in Top-1000} & No & \textbf{Yes} \\
\textbf{NDCG} & 0.0000 & \textbf{0.6667} \\
\textbf{Top-5 Retrieved} &
— &
\texttt{889753004} (target): Fazer balloon sweater, boxy cropped shape \\
\bottomrule
\end{tabular}
\caption{Example 3 highlights OSPO's ability to infer latent color preference shifts (neutral $\rightarrow$ pastel) while grounding in wardrobe context, enabling successful retrieval. Example 4 demonstrates style attribute transfer from a reference item (Jennifer blazer) to recover the correct product category and fit, where GRPO fails due to generic or weakly grounded expansions.}
\label{tab:grpo_ospo_case_study_2}
\end{table*}

\begin{table*}[t]
\centering
\small
\setlength{\tabcolsep}{5pt}
\renewcommand{\arraystretch}{1.1}
\begin{tabular}{p{0.20\linewidth} p{0.40\linewidth} p{0.40\linewidth}}
\toprule
\textbf{Field} & \textbf{GRPO} & \textbf{\ospoProp{}} \\
\midrule
\multicolumn{3}{l}{\textit{Example 1 — Attribute enrichment}} \\
\midrule
\textbf{User Query} &
\multicolumn{2}{p{0.80\linewidth}}{all dog man books set} \\
\textbf{Expanded Query} &
``All Dog Man books set by Dav Pilkey including all volumes in the series.'' &
All sets of Dog Man books that include the complete series. Specifically, a collection containing all Dog Man books by Dav Pilkey, ideally in paperback format. A boxed set or hardcover edition would also be great, with durable covers suitable for kids. \\
\textbf{Top-5 Retrieved} &
— &
\texttt{9123877294}, \texttt{B08462HBKH}, \texttt{9123786302}, \texttt{1338323210}, \texttt{1338535625} \\
\textbf{NDCG} & 0.000 & \textbf{1.000} \\
\midrule
\multicolumn{3}{l}{\textit{Example 2 — Brand retention}} \\
\midrule
\textbf{User Query} &
\multicolumn{2}{p{0.80\linewidth}}{actloe women cowl neck sweater} \\
\textbf{Expanded Query} &
Women's cowl neck sweater made of soft material suitable for fall, available in sizes S, M, and L. &
Actloe women's cowl neck sweater in soft wool blend material with a cozy texture, perfect for layering during cooler seasons. Features a fashionable cowl neckline in various flattering colors. \\
\textbf{Top-5 Retrieved} &
— &
\texttt{B07TYXRMZH} (target) \\
\textbf{NDCG} & 0.000 & \textbf{1.000} \\
\bottomrule
\end{tabular}
\caption{ESCI query expansion: cases where \ospoProp{} outperforms GRPO. 
Example 1 shows OSPO adding format-specific attributes (paperback, boxed set, 
durable covers) that GRPO omits, retrieving all 9 targets vs.\ none. 
Example 2 shows OSPO retaining the brand name \textit{Actloe}—dropped by 
GRPO—as the key retrieval signal.}
\label{tab:esci_ospo_wins}
\end{table*}

\begin{table*}[t]
\centering
\small
\setlength{\tabcolsep}{5pt}
\renewcommand{\arraystretch}{1.1}
\begin{tabular}{p{0.20\linewidth} p{0.40\linewidth} p{0.40\linewidth}}
\toprule
\textbf{Field} & \textbf{GRPO} & \textbf{\ospoProp{}} \\
\midrule
\multicolumn{3}{l}{\textit{Example 3 — Category hallucination (GRPO wins)}} \\
\midrule
\textbf{User Query} &
\multicolumn{2}{p{0.80\linewidth}}{brother se600 hoops} \\
\textbf{Expanded Query} &
Brother SE600 embroidery machine with hoops for high-quality embroidery projects. &
Brother SE600 Hoops Inkjet Printer — high-quality printer for home and small office use, 
featuring print speeds up to 25 PPM and support for a wide range of media sizes. \\
\textbf{Top-5 Retrieved} &
\texttt{B07YVJ14K7}, \texttt{B07DD1C6N2}, \texttt{B07GPMGPBM}, \texttt{B07H57ZX8L}, 
\texttt{B07P9P1HW6} &
— \\
\textbf{NDCG} & \textbf{1.000} & 0.000 \\
\midrule
\multicolumn{3}{l}{\textit{Example 4 — Both succeed with different strategies (tie)}} \\
\midrule
\textbf{User Query} &
\multicolumn{2}{p{0.80\linewidth}}{german marzipan chocolate} \\
\textbf{Expanded Query} &
German-made marzipan chocolate bar or candy, specifically looking for a high-quality 
German marzipan chocolate product. &
German marzipan chocolate confections with chocolate coating, made from almond paste, 
sugar, and egg whites, originating from regions like Nuremberg or Dresden. Various 
shapes with smooth chewy texture and rich chocolate exterior. \\
\textbf{Top-5 Retrieved} &
\texttt{B07JCPBXJG}, \texttt{B001FA1CCO}, \texttt{B086PDVJSL}, \texttt{B004H0ZCI4}, 
\texttt{B00A7J3MBS} &
\texttt{B07JCPBXJG}, \texttt{B086PDVJSL}, \texttt{B001FA1CCO}, \texttt{B004H0ZCI4}, 
\texttt{B00A7J3MBS} \\
\textbf{NDCG} & \textbf{1.000} & \textbf{1.000} \\
\bottomrule
\end{tabular}
\caption{ESCI product search: GRPO win and tie cases. Example 3 shows OSPO's 
failure mode—a verbose generation causes a category hallucination (embroidery machine 
$\to$ inkjet printer), sending retrieval to the wrong embedding region while GRPO's 
concise expansion correctly identifies the product. Example 4 shows both methods 
succeeding via different strategies: GRPO with a concise factual phrase, OSPO with 
regional provenance and ingredient detail—consistent with the finding that query 
length does not predict retrieval performance ($r{=}0.026$, $p{=}0.401$).}
\label{tab:esci_grpo_win_tie}
\end{table*}

\clearpage
 
\end{document}